\documentclass{article}

\usepackage{microtype}
\usepackage{graphicx}
\usepackage{subfigure}
\usepackage{booktabs} % for professional tables
\usepackage{mathrsfs}

\usepackage{hyperref}
\usepackage{subcaption}
\usepackage{listings}
\usepackage{xcolor}
\usepackage{multirow}
\usepackage[accepted]{icml2024}

\usepackage{amsmath}
\usepackage{amssymb}
\usepackage{mathtools}
\usepackage{amsthm}

\usepackage[capitalize,noabbrev]{cleveref}

\theoremstyle{plain}
\newtheorem{theorem}{Theorem}[section]

\newtheorem{corollary}[theorem]{Corollary}
\theoremstyle{definition}
\newtheorem{definition}[theorem]{Definition}

\theoremstyle{remark}

\usepackage[textsize=tiny]{todonotes}

\icmltitlerunning{Improving Open-Ended Text Generation via Adaptive Decoding}

\begin{document}

\twocolumn[
\icmltitle{Improving Open-Ended Text Generation via Adaptive Decoding}



\icmlsetsymbol{equal}{*}

\begin{icmlauthorlist}
\icmlauthor{Wenhong Zhu}{yyy}
\icmlauthor{Hongkun Hao}{yyy}
\icmlauthor{Zhiwei He}{yyy}
\icmlauthor{Yiming Ai}{yyy}
\icmlauthor{Rui Wang}{yyy}
\end{icmlauthorlist}

\icmlaffiliation{yyy}{MT Lab, Department of Computer Science and Engineering
, Shanghai Jiao Tong University, Shanghai, China}

\icmlcorrespondingauthor{Rui Wang}{wangrui12@sjtu.edu.cn}

\icmlkeywords{Machine Learning, ICML}

\vskip 0.3in
]



\printAffiliationsAndNotice{}  % leave blank if no need to mention equal contribution

\begin{abstract}
Current language models decode text token by token according to probabilistic distribution, and determining the appropriate candidates for the next token is crucial to ensure generation quality. This study introduces \textit{adaptive decoding}, a mechanism that dynamically empowers language models to ascertain a sensible candidate set during generation.  Specifically, we introduce an entropy-based metric called \textit{confidence} and conceptualize determining the optimal candidate set as a \textit{confidence-increasing} process. The rationality of including a token in the candidate set is assessed by leveraging the increment of confidence. Experimental results reveal that our method balances diversity and coherence well. The human evaluation shows that our method can generate human-preferred text. Additionally, our method can potentially improve the reasoning ability of language models. 
\end{abstract}
\section{Introduction}
When it comes to open-ended text generation tasks, the aim is to produce coherent and fluent output based on limited input information \cite{holtzman2019curious}. These tasks span various applications, including casual conversation \cite{thoppilan2022lamda}, storytelling \cite{mostafazadeh-etal-2016-corpus}, and similar domains. Language models (LMs) presently estimate the likelihood of the next token in text generation, and they have gained extensive application in various tasks \cite{brown2020language, touvron2023llama}. One factor that determines the quality of generation is the decoding algorithm.

\begin{figure*}
  \centering
  \subfigure[The probability of human-preferred token and LM-preferred token.]{
    \includegraphics[width=5cm]{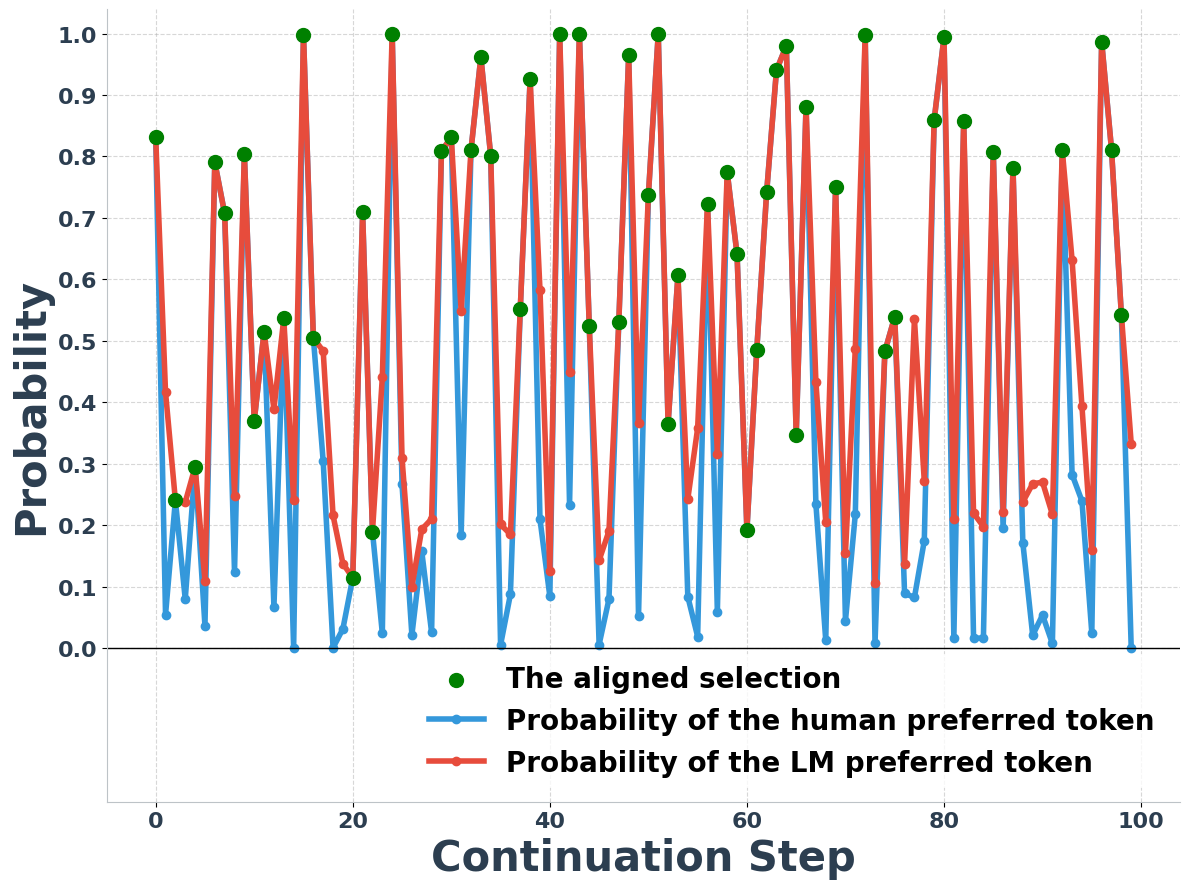}
    \label{fig:generation_1}
  }
  \subfigure[Head distribution given the \textbf{prefix}: \textit{``Growing up, Obama faced the \underline{\textcolor{red}{challenges}}''}]{
    \includegraphics[width=5cm]{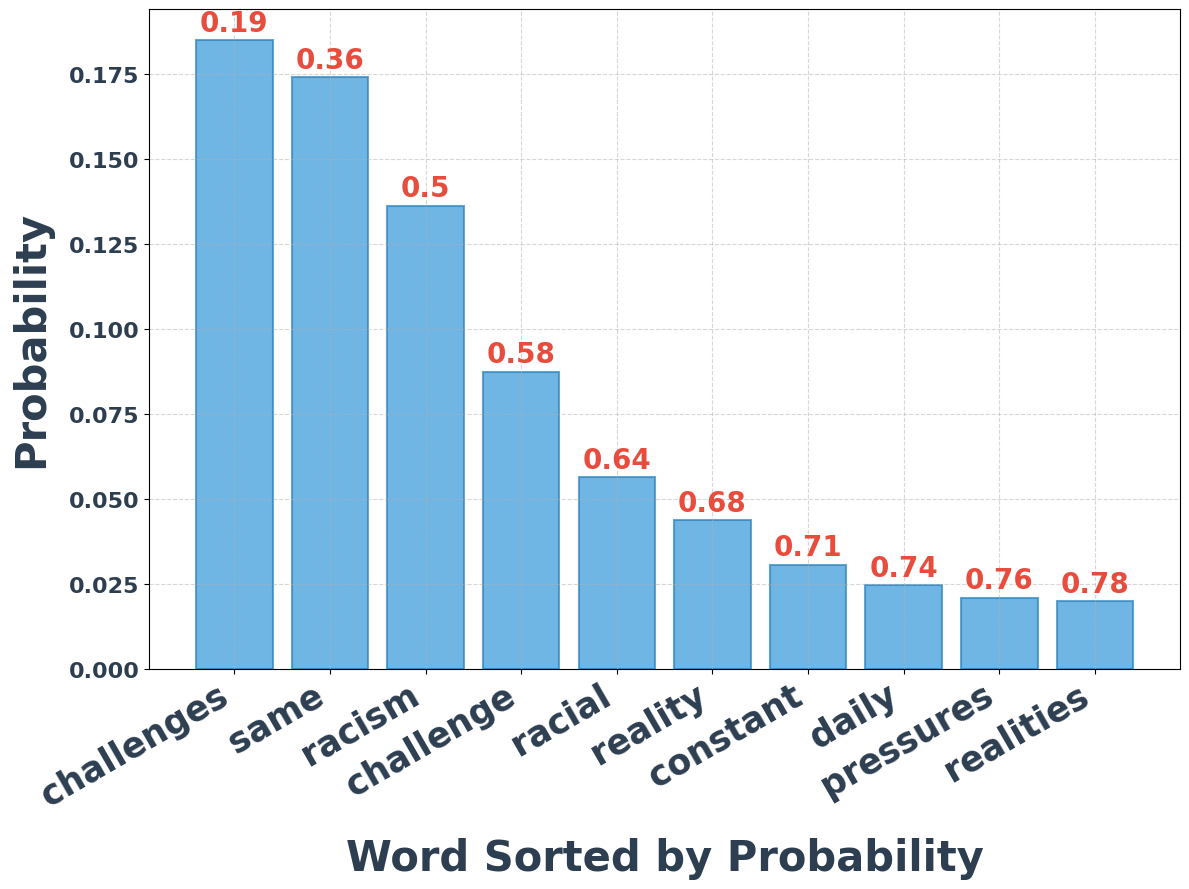}
    \label{fig:generation_2}
  }
  \subfigure[Head distribution given the \textbf{prefix}:  \textit{``attended Columbia University in \underline{\textcolor{red}{New}}''}]{
    \includegraphics[width=5cm]{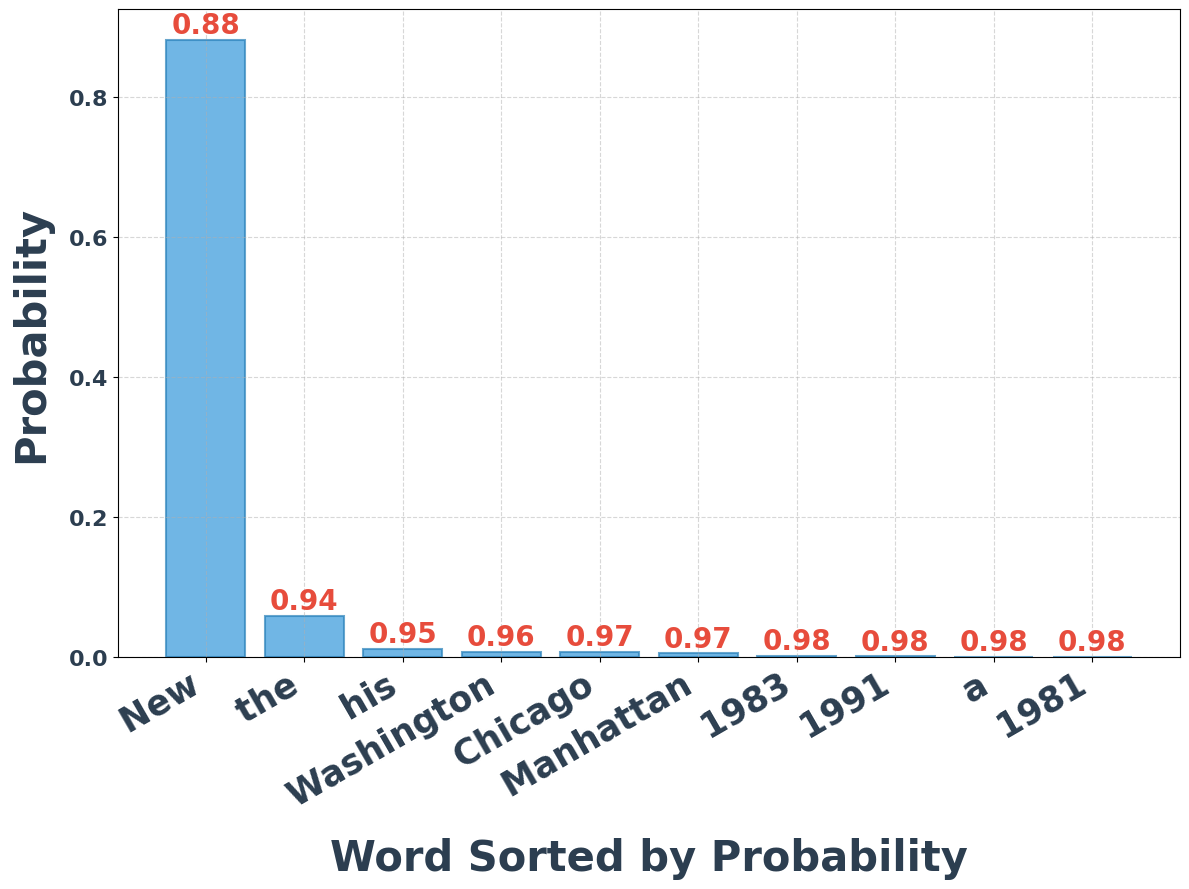}
    \label{fig:generation_3}
}
  \caption{Human written text: ``\textit{Barack Obama was born in 1961. He was raised in Hawaii by his mother and grandparents. Growing up, Obama faced the challenges of being biracial, with a Kenyan father and an American mother. Despite these challenges, he excelled academically and eventually attended Columbia University in New York City.}'' %After completing his education, Obama worked as a community organizer in Chicago, focusing on issues such as job training and housing for low-income communities. His passion for social justice and community development propelled him into the world of politics. 
  We provide this human-written text for GPT2-XL and use teacher-forcing decoding. 
  }
  \label{fig:generation}
\end{figure*}

For LMs, training solely on maximum likelihood estimation (MLE) without additional fine-tuning, such as human preference alignment, often raises distinct degradation issues~\cite{holtzman2019curious}. These challenges are closely intertwined with decoding algorithms. The decoding algorithm can be seen as a mechanism to truncate the entire vocabulary space into the next candidate token set. When this set is limited, as seen in greedy decoding and beam search, it leads to heightened repetition as the model favors previously decoded tokens \cite{li-etal-2020-dont, zhu-etal-2023-penalty}. Conversely, when this set is expansive, as in top-$p$ sampling, it can result in incoherence and hallucination \cite{xu-etal-2023-look}. This occurs because certain improbable tokens still contribute to the cumulative probability, causing issues when sampling these tokens. Numerous efforts have been devoted to addressing these challenges. Some endeavors focus on enhancing diversity within the specified sampling space \cite{su2022contrastive, xu-etal-2023-look}, while others strive to identify a rational sampling space \cite{meister2023locally, hewitt2022truncation}.

We provide one human-written text example for the GPT2-XL model and use teacher-forcing decoding to demonstrate the distinction between the human-preferred and LM-preferred tokens. In doing so, we save the probability of human-preferred and LM-prefered tokens. As shown in \cref{fig:generation_1}, the choices made by humans and the LM predominantly align at high probability points. The distinction lies in that the portion of the token selected by humans tends to fall within the low probability region predicted by the LM. Simultaneously, the model tends to choose tokens with relatively higher probability. This suggests that the selection of the next token could have many possibilities. 

On the other hand, the distribution predicted by the LM exhibits a long-tailed pattern \cite{holtzman2019curious}, indicating that choosing the next token following the previous context is constrained. The main difference in the long tail distribution is the head shape. For example, as shown in \cref{fig:generation_2}, the distribution is flattened, indicating that the LM has multiple choices of the next token. The diversity of generated text may be compromised if our candidate set is too limited. While in \cref{fig:generation_3}, the distribution is on a sharp shape, implying the restricted choices. In this scenario, an excessively large candidate set could lead to a loss of coherence in the generated text. These observations raise the question: how do we determine the rational size of the candidate set to ensure both diversity and coherence in the generated text?

In this study, we present an \textit{adaptive decoding} algorithm designed to dynamically determine the size of a suitable candidate set during generation. Entropy is a metric to quantify the disorder or randomness within a system. It would be useful in gauging the confidence of an LM in predicting the probability distribution of the next token. We conceptualize identifying the optimal candidate set as an \textit{confidence-increasing} procedure. Specifically, we establish a system embracing two sets: a candidate set, initially empty, and an ordered set, initially encompassing the entire vocabulary. One metric named \textit{confidence} was formulated to gauge system confidence based on entropy. Continuously selecting tokens with the highest probability from the ordered set and appending them to the candidate set significantly enhances the confidence of this system. The rationality of a token to join the candidate set is assessed based on the increment of confidence, allowing the model to dynamically determine a suitable set of candidates.

In our experiments, we performed two open-ended text generation tasks: document continuation and story generation. The results suggest that our approach significantly enhances diversity while preserving coherence in GPT2-XL(1.5B) and Llama2-7B models. On Llama2-7B-chat, coherence has been enhanced alongside increased diversity. Additionally, two mathematical reasoning tasks demonstrate that our method potentially improves the reasoning ability of language models. Human evaluation results show that our approach produces text that is more human-preferred. 

In summary, the main contributions of this paper are as follows: 
\begin{itemize}
    \item We integrate the concept of entropy reduction into the text decoding process and design a confidence increment as the metric to assess the rationality of a token as a candidate. 
    \item We introduce a novel \textit{adaptive decoding} mechanism that enables LMs to determine an appropriately sized set of candidates autonomously during the generation. \footnote{The code is available at \href{https://github.com/zwhong714/adaptive_decoding}{https://github.com/zwhong714/adaptive\_decoding}.}
    \item \textit{Adaptive decoding} effectively generates text in a style resembling human expression, ensuring a balance between coherence and diversity.
\end{itemize}

\section{Problem Formulation}
\subsection{Language Models}
We follow the definitions given in \citet{hewitt2022truncation}. Let the random variable $\boldsymbol{X} = (X_1, X_2, ..., X_T)$ represents a sequence of tokens, where each $X_t$ belongs to a finite vocabulary $\mathcal{V}$. 
% We use $\boldsymbol{x}_{<t}$ to indicate a given prefix, $x_i$ as a specific word in $\mathcal{V}$, and $x$ as a random token in $\mathcal{V}$.
We use $\boldsymbol{x}_{<t}$ to indicate a given prefix, $x_t$ as a specific word in $\mathcal{V}$, and $x$ as a random token in $\mathcal{V}$. 
An auto-regressive LM is a distribution $P_\theta(\boldsymbol{X})$ with parameters $\theta$, which is factorized as $P_\theta(x) = \prod_{t=1}^T P_\theta(x_t | \boldsymbol{x}_{<t})$. We denote $P_\theta(X_t | \boldsymbol{x}_{<t})$ as the conditional distribution of the LM given the context $\boldsymbol{x}_{<t}$ over $\mathcal{V}$. The standard method for training an LM using maximum likelihood estimation of the training data $\mathcal{D}$, where log-likelihood is defined as:
\begin{equation}
\label{log_likelihood}
    \mathcal{L}(\boldsymbol{\theta})=\sum_{\boldsymbol{X} \in \mathcal{D}} \log P_{\boldsymbol{\theta}}(\boldsymbol{X}).
\end{equation}

This is also equivalent to minimizing the forward cross-entropy between the true distribution $P^*(\boldsymbol{X})$ and $P_\theta(\boldsymbol{X})$: 
\begin{equation}
    -E_{\boldsymbol{X}\sim P^*}\sum_{t=1}^{|\boldsymbol{X}|} P^*(x|\boldsymbol{x}_{<t})\log{P_{\theta}(x|\boldsymbol{x}_{<t})} 
\end{equation}
And recent LMs have achieved remarkably low forward cross-entropy \cite{radford2019language, touvron2023llama}.

\paragraph{Recall-prioritization}
\label{recall}
However, model training with MLE has inherent limitations for matching model distribution and that of human language \cite{meister2023locally}. This is essentially because cross-entropy is recall-prioritized \cite{meister-etal-2023-efficacy}. Recall here is defined as the coverage of the model distribution $P_\theta$ over the true distribution $P^*$, indicating that a high recall implies tokens with high likelihood under $P^*$ also have high likelihood under $P_\theta$. In other words, cross-entropy focuses on increasing the model likelihood of the ground-truth next token \cite{ren2024emo}. Therefore, according to \cref{log_likelihood}, $P_\theta(x_t|\boldsymbol{x}_{<t})$ must be nonzero for all tokens $x_t$ in every string $X$ in the training dataset $\mathcal{D}$, which leads to multiple possible candidates for the next token. Besides, label smoothing has become a de-facto standard for large-scale training. In practice, there is more than one target next token during training due to the entropy regularization applied to the loss. As indicated in \cref{well-trained}, the model prediction for the next token may not be precisely identical to the teacher-forced token. Nevertheless, it is essential to recognize that the model exhibits a significant capability in predicting the next token.

\begin{table}[t]
\caption{Preliminary experiment: HIT@$k$ is a metric that measures the human preferred token falls into the top-$k$ space of an LM predicted distribution. Results obtained by 1000 human-written text from WikiText-103.}
\label{well-trained}
\vskip 0.15in
\begin{center}
\begin{small}
\begin{sc}
\begin{tabular}{lcccr}
\toprule
Model & hit@1 & hit@3 & hit@5 \\
\midrule
gpt2   & 32.37 & 49.89 & 57.01  \\
gpt2-xl    & 37.09 & 55.21  &    62.46     \\
Llama2-7B    & 47.14 & 65.40 & 72.08   \\
Llama2-chat-7B    & 45.04 & 62.81 & 69.57  \\
\bottomrule
\end{tabular}
\end{sc}
\end{small}
\end{center}
\vskip -0.1in
\end{table}

\paragraph{Text generation}
LMs can assess the likelihood of existing sequences and produce new tokens through the expression $x \sim P_\theta(\boldsymbol{X})$. The common application involves the generation of open-ended text. When presented with a sequence of $m$ tokens extracted from natural language, denoted as $\mathcal{C} = \{x_1, ..., x_m\}$, serving as context, the LM aims to construct an $n$-tokens extension. This extension is decoded using the probability distribution predicted by the LM: $P_\theta(x_{m+1}:x_{m+n}|\mathcal{C}) = \prod_{t=1}^n P_\theta(x_t|\mathcal{C}, x_{m+1}, ..., x_{m+t-1})$. The generation of the continuation occurs token by token, employing a specific decoding strategy.

\subsection{Degeneration Problems}

\paragraph{Repetition}
Repetition often arises due to models trained using MLE without any intervention, leading to a gradual increase in the probability of tokens that have already occurred \cite{xu2022learning, zhu-etal-2023-penalty}. This tendency becomes particularly pronounced when utilizing greedy decoding or beam search~\cite{holtzman2019curious}. While this behavior aligns with the correct objective of MLE-trained models to continually maximize likelihood in the text generation process, it falls short of generating text preferred by humans. As illustrated in \cref{degeneration}, this issue appears to be independent of the size of the model parameters. Even the GPT-3 base model(175B) exhibits significant repetition problems when using greedy decoding.

\paragraph{Incoherence}
Many decoding algorithms have been proposed to truncate low probability tokens directly and have proven to be useful \cite{holtzman2019curious, hewitt2022truncation, meister2023locally}. They typically compute the following truncated distribution at each time step:

\begin{equation}
\label{eq: trunc}
    P_{\text {trunc }}\left(x \mid \boldsymbol{x}_{<t}\right)= \begin{cases}P_\theta\left(x \mid \boldsymbol{x}_{<t}\right) / Z_{\boldsymbol{x}_{<t}} & x \in \mathcal{A}_{\boldsymbol{x}_{<t}} \\ 0 & \text { o.w, }\end{cases}
\end{equation}

where we denote $\mathcal{A}_{\boldsymbol{x}_{<t}} \subseteq \mathcal{V}$ as the allowed set comprising candidate following tokens for a given prefix, and $Z_{\boldsymbol{x}_{<t}} = \sum_{x \in \mathcal{A}_{\boldsymbol{x}_{<t}}} P_{\theta}(x|\boldsymbol{x}_{<t})$ is the renormalization term.

Effective truncation is the key to improving the quality of text generation. Otherwise, there is a risk of incoherent and hallucinatory outputs, as the model may sample from low-probability tokens that may not be related to the prefix. As we can see from \cref{degeneration}, the text is highly incoherent with the text generation going on when we use top-$p$ sampling \cite{holtzman2019curious}.

\begin{table}[t]
\caption{Degeneration problem: repetition and incoherence}
\label{degeneration}
\vskip 0.15in
\begin{center}
\begin{small}
\begin{sc}
\begin{tabular}{p{0.45\textwidth}}
\toprule
\textnormal{\textbf{GPT3 base model (davici-002), greedy}} \\
\midrule
\textnormal{\textbf{Prefix}: \textit{``Barack Obama was born in Honolulu, Hawaii. He was born in"}}\\\midrule
\textnormal{\textbf{(Repetition) Continuation}: \textit{``1961. He was born on August 4, 1961. Barack Obama was born in Honolulu, Hawaii on August 4, 1961. Barack Obama was born in Honolulu" }}\\
\midrule
\textnormal{\textbf{GPT2-XL base model, top-$\boldsymbol{p=0.95}$}} \\
\midrule
\textnormal{\textbf{Prefix}: \textit{``Barack Obama was born in Honolulu, Hawaii. He was born in"}}\\\midrule
\textnormal{\textbf{(Incoherence) Continuation}: \textit{``Hawaii, and he is a legitimate U.S. citizen. He's the president, and I'm the president of the United States." — Nov. 14, 2009"}}\\
\bottomrule
\end{tabular}
\end{sc}
\end{small}
\end{center}
\vskip -0.1in
\end{table}

\section{Preliminaries}
\label{pre}

\begin{definition} 
\label{def: entropy} 
Given a discrete random variable $X$,  which takes values in the vocabulary $\mathcal{V}$ and is distributed according to $p: P_\theta(X) \rightarrow [0, 1]$, the \textit{entropy} is defined as 
$$\mathrm{H}(X):=-\sum_{x \in \mathcal{V}} p(x) \log p(x).$$
\end{definition}

\begin{theorem}
\label{thm: upper}
For any random variable $X$, 
$$0 \le H(X) \le \log |\mathcal{V}|,$$
where $|\mathcal{V}|$ denotes the size of the vocabulary $\mathcal{V}$. The upper bound is tight if and only if $X$ is distributed uniformly on $\mathcal{V}$. The lower bound is tight if and only if $X$ is deterministic. 
\end{theorem}

By \cref{thm: upper}, the upper bound is tight when token distribution predicted by the LM is uniform, which means each token in the vocabulary $\mathcal{V}$ has the same probability of being sampled. However, the upper bound is often impossible to achieve in the real world. Moreover, the true upper bound is frequently linked to the model's capability to predict the next token.

To mitigate the impact of the upper bound, we normalized entropy using the min-max approach. (1) This ensures that all changes fall within a similar range, enabling the comparison and analysis of distributions predicted by different models. (2) It also facilitates comparing and analyzing entropy reduction processes.

\begin{definition} 
\label{def:confidence}
For any random variable $X$, the \textit{confidence} $\text{Conf}(X)$ is defined as the min-max scaling of entropy, and the range is $[0, 1]$:
$$\text{Conf}(X):= 1 + \frac{\sum_{x \in \mathcal{V}} p(x) \log p(x)}{\log |\mathcal{V}|}.$$
\end{definition}

The expression $\sum_{x \in \mathcal{V}} p(x) \log p(x)$ can be conceptually divided into two components. The first component is the known part $\sum_i^{k} p_i \log p_i$, where $k$ represents the number of known items. The second component is the unknown part, introducing maximum uncertainty, expressed as $({1 - \sum_{i}^{k}p_i}) \log\frac{1 - \sum_{i}^{k}p_i}{|\mathcal{V}|-k}$. Consequently, we can define a pivot $k$ to divide the distribution $P_\theta(X)$ based on the number of known items. This metric related to $k$ is as follows:
\begin{equation}
\label{conf}
\begin{aligned}
    \text{Conf}_k(X) = 1 + &\frac{1}{\log|\mathcal{V}|}(\sum_i^{k} p_i \log p_i \\
    & + ({1 - \sum_{i}^{k}p_i}) \log\frac{1 - \sum_{i}^{k}p_i}{|\mathcal{V}|-k} ),
\end{aligned}
\end{equation}
where $p_i$ is the abbreviation of $p(x_i)$.

\subsection{Conf$_k(X)$ Is Monotonically Increasing}
\label{increase}
We define the rationale for including a token in the candidate set as the extent to which the presence or absence of this token affects the model's confidence. Specifically, it is defined as the difference in \cref{conf} between neighboring pivots.
\begin{equation}
\label{eq: Deltaconf}
\begin{aligned}
&\Delta \text{Conf} =\frac{1}{\log|\mathcal{V}|} (\text{Conf}_k(X)- \text{Conf}_{k-1}(X))
\\&= \frac{1}{\log|\mathcal{V}|} ( p_k\log p_k + (1 - \sum_i^k p_i) \log\frac{(1 - \sum_i^k p_i)}{|\mathcal{V}| - k} \\
&- (1 - \sum_i^{k-1} p_i) \log\frac{(1 - \sum_i^{k-1} p_i)}{|\mathcal{V}| - k + 1}).
\end{aligned}
\end{equation}

In \cref{proof}, we demonstrate that \cref{eq: Deltaconf} is greater than zero when $P_\theta(X)$ is sorted in decreasing order. Under this condition, confidence steadily increases as $k$ continues to grow.

\subsection{Descending Trend in $\Delta$Conf}
\label{dconf}
Although $\Delta$Conf may lack a clear monotonic property, assessing its variability is feasible by examining its upper and lower bounds. Utilizing \cref{eq:ineq1}, we can establish the lower bound as $\Delta \text{Conf} \ge \frac{p_k}{\log|\mathcal{V}|} \log(|\mathcal{V}| - k + 1)$, while employing both \cref{eq:cor} and \cref{eq:ineq1} allows us to derive the upper bound, namely $\Delta \text{Conf} \le \frac{(1 - \sum_i^{k - 1} p_i)}{\log|\mathcal{V}|} \log(|\mathcal{V}| - k + 1)$.

As the number of known items $k$ increases, it is observed that the probability $p_k$ decreases. Additionally, both $\log(|\mathcal{V}| - k + 1)$ and $(1 - \sum_i^{k - 1} p_i)$ exhibit a decreasing trend. Ultimately, the upper and lower bounds converge to zero. This implies that $\Delta$Conf also demonstrates a decreasing trend, a pattern influenced by the interplay between the probability $p_k$ and the number of known items $k$.

\subsection{Approximation by Truncation}
\label{approximation}
\begin{figure}[ht]
\vskip 0.2in
\begin{center}
\centerline{\includegraphics[width=\columnwidth]{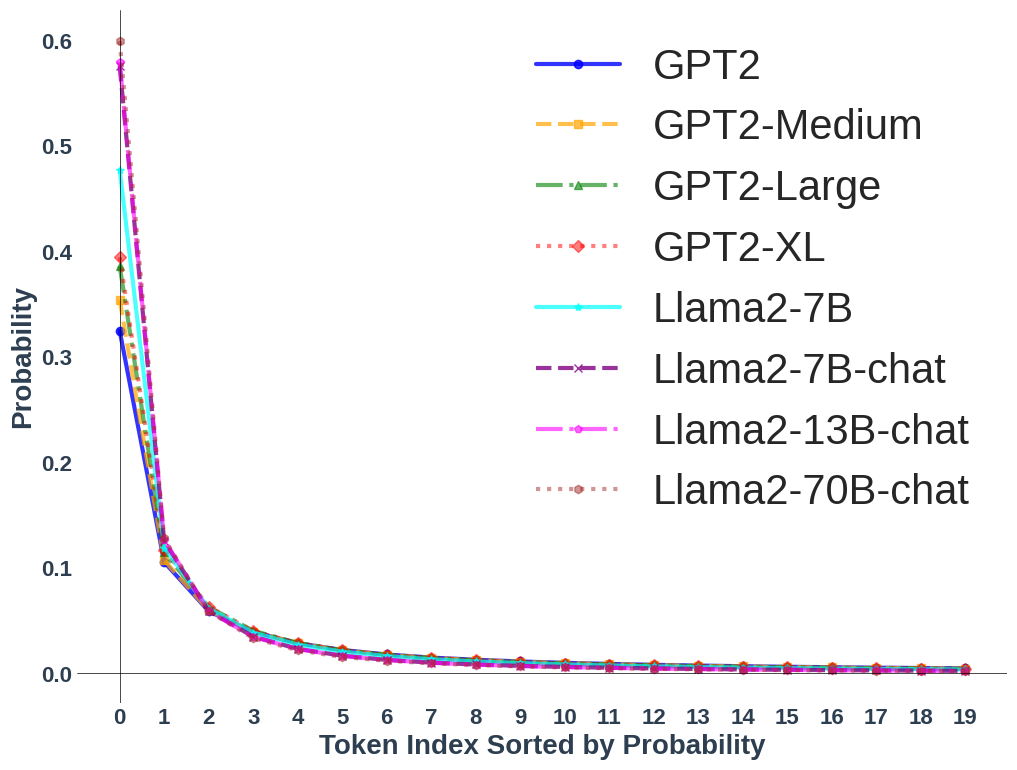}}
\caption{Preliminary experiment: Averaged top-20 token distribution of different LMs on 1000 human-written texts from WikitText-103.}
\label{fig:long_tailed}
\end{center}
\vskip -0.2in
\end{figure}

% \begin{definition}
% \label{def:long-tailed}
% For any random variable $X$, it is considered to follow a long-tailed distribution if its distribution exhibits the following limit:
% $$
% \lim _{x \rightarrow \infty} \operatorname{Pr}[X>x+m \mid X>x]= 1,
% $$
% \end{definition}
% where $m$ is a constant. 

As illustrated in \cref{fig:long_tailed}, the token distribution predicted by the LM exhibits a pronounced long-tailed pattern. Notably, this phenomenon becomes more conspicuous as the LM parameters increase and a human-alignment operation follows. %Our observation reveals that the long-tail distribution predicted by the LM typically manifests within a finite value $k$, such that $P_\theta(X>k + m|X>k) \rightarrow 1$.

\begin{corollary}
\label{cor: truncation}
When a distribution exhibits a long-tail shape, the entropy of the truncated distribution $\mathrm{H}(X(k))$ is approximately equal to the original entropy and an error within a smaller range of $\epsilon$.
\begin{equation}
\begin{aligned}
\label{trunc}
\mathrm{H}(X) = \mathrm{H}(X(k)) + \epsilon,
\end{aligned}
\end{equation}
where $X(k)$ follows the truncated distribution as defined by the \cref{eq: trunc}.
\end{corollary}

\cref{trunc} shows we can determine a reasonable candidate set size $k$ by setting a reasonable threshold $\epsilon$.

\renewcommand{\algorithmicrequire}{\textbf{Input:}}
\renewcommand{\algorithmicensure}{\textbf{Output:}}
\newcommand*\ANNOTATE[1]{\hfill\(\triangleright\) #1}%

\section{Adaptive Decoding}
Our method adheres that the optimal threshold must balance diversity (i.e., including as many tokens as possible in the candidate set) and coherence (i.e., avoiding the inclusion of tokens outside the true support) \cite{finlayson2024closing}.

\begin{figure}[ht]
\vskip 0.2in
\begin{center}
\centerline{\includegraphics[width=\columnwidth]{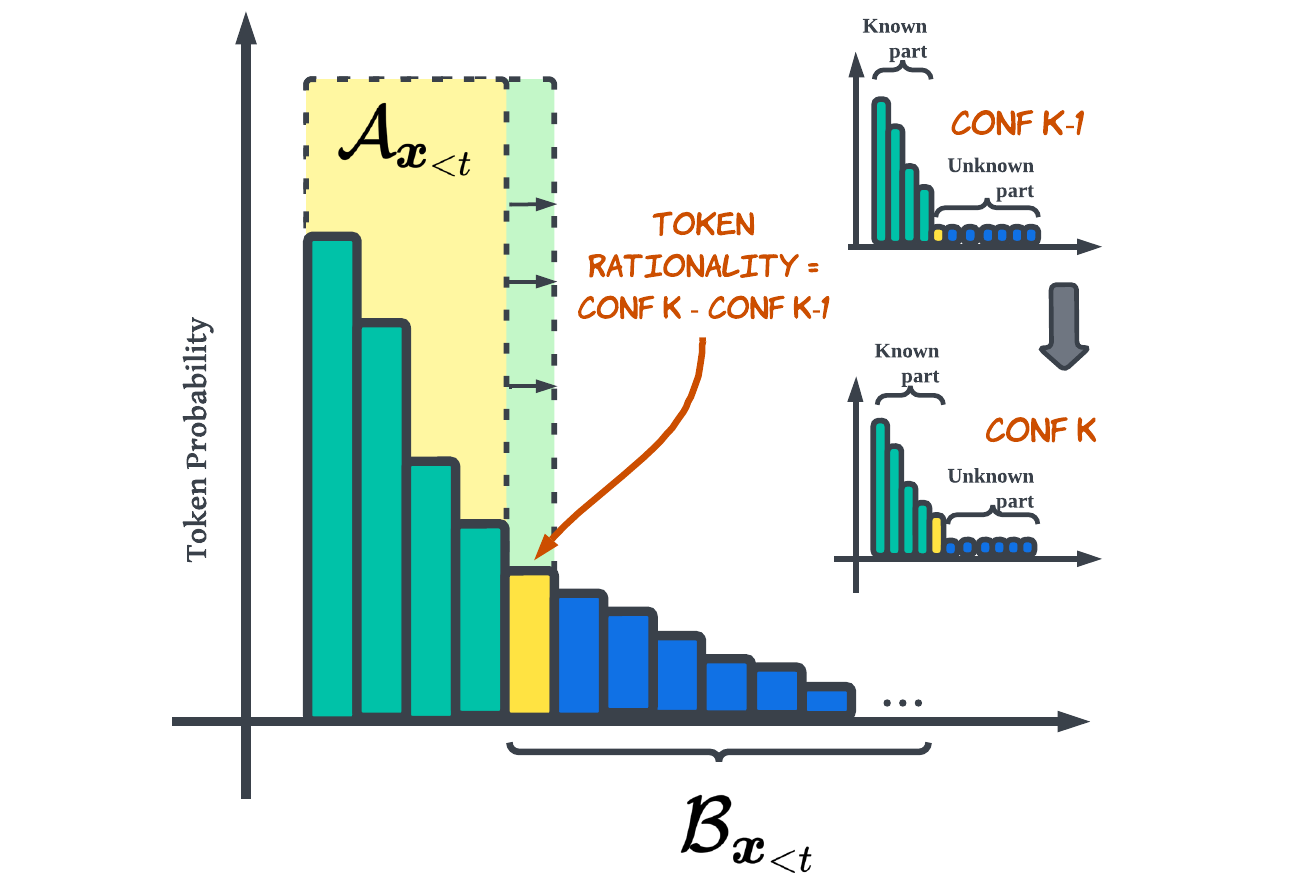}}
\caption{How adaptive decoding works. Expanding the candidate set $\mathcal{A}_{\boldsymbol{x}_{<t}}$ depends on the rationality of the token. This rationality is determined by the extent to which the model's confidence increases when considering the inclusion of this token.}
\label{fig: visualization}
\end{center}
\vskip -0.2in
\end{figure}

\paragraph{Initialization}  For clarity, we define the known portion, as discussed in \cref{pre}, as the next token candidate set $\mathcal{A}_{\boldsymbol{x}_{<t}}$, and the unknown part as $\mathcal{B}_{\boldsymbol{x}_{<t}}$. The initialization involves setting $\mathcal{A}_{\boldsymbol{x}_{<t}}$ to an empty set and $\mathcal{B}_{\boldsymbol{x}_{<t}}$ to an ordered collection of the entire vocabulary. The order is based on the descending probabilities predicted by the LM.

\paragraph{Truncation} As illustrated in \cref{fig: visualization}, identifying a suitable set of candidates can be viewed as an iterative procedure. This involves transferring the token with the highest probability from $\mathcal{B}_{\boldsymbol{x}_{<t}}$ to $\mathcal{A}_{\boldsymbol{x}_{<t}}$. Subsequently, we calculate the incremental change in the Conf metric, as defined in \cref{eq: Deltaconf}. This increment serves as a measure of the rationality of a token within the candidate set. The iteration continues until the rationality falls below a pre-set threshold, denoted as $\epsilon$. At this point, we finalize $\mathcal{A}_{\boldsymbol{x}_{<t}}$ as our truncation space. %It is reasonable to truncate by $\Delta$Conf because, in \cref{eq: Deltaconf}, we demonstrate that the increment is linked to the token probability and the size of the current candidate set. This ensures that the truncated space reflects the degree of confidence before truncation. 

\paragraph{Computational complexity}
Efficient implementation of adaptive decoding is achievable through \cref{algorithm}. Firstly, sorting words by their predicted probabilities can be accomplished in $\mathcal{O}(|\mathcal{V}|\log|\mathcal{V}|)$. Secondly, the calculation of the increment $\Delta$Conf, as described in \cref{eq: Deltaconf}, is a $\mathcal{O}(|\mathcal{V}|)$ operation. By getting $\Delta$Conf, we can effectively identify the last index with a value below the pre-set threshold in $\mathcal{O}(\log|\mathcal{V}|)$. In summary, our algorithm exhibits a time complexity of $\mathcal{O}(|\mathcal{V}|\log|\mathcal{V}|)$. However, we can fully leverage the parallel computing power of GPUs to accelerate the computation process.

\begin{algorithm}[H]
\caption{Adaptive Decoding Algorithm}
\begin{algorithmic}
\label{algorithm}
\REQUIRE Prefix $\mathcal{C} = \{x_1, \dots, x_m\}$; Language Model $P_{\theta}$; Vocabulary size $|\mathcal{V}|$; Threshold $\epsilon$
\ENSURE Continuation $\mathcal{G} = \{x_{m+1}, x_{m+2}, \dots, x_{m+n}\}$

\STATE $\mathcal{G} \gets \{\}$

\FOR{$t = m+1$ {\bfseries to} $m+n$}
    \STATE $p \gets P_{\theta}(x_t|\mathcal{C})$ 
    \STATE $p \gets \text{sort}(p, \text{decreasing order})$
    \STATE $\text{cumsum} \gets \text{cumulative sum of } p$
    \STATE $\text{residual\_index} \gets |\mathcal{V}| - \text{range}(1, |\mathcal{V}| + 1)$
    
    \STATE $term1 \gets p \cdot \log\left(\frac{p \cdot \text{residual\_index}}{1 - \text{cumsum}}\right)$
    \STATE $term2 \gets \log\left(\frac{1 - \text{cumsum}}{\text{residual\_index}}\right) - \log\left(\frac{1 - \text{cumsum} + p}{\text{residual\_index} + 1}\right)$
    \STATE $\Delta \text{Conf} \gets \frac{term1 + (1 - \text{cumsum} + p) \cdot term2}{\log |\mathcal{V}|}$
    
    \STATE $k \gets \max(1, \text{LastIndex}(\Delta \text{Conf} > \epsilon))$
    \STATE \textcolor{green!60!black}{$\triangleright$ top-$k$ sampling}
    \STATE $x_t \gets \text{sample from } \mathcal{V}^{(k)}$ 
    \STATE $\mathcal{G} \gets \mathcal{G} \cup \{x_t\}$ 
\ENDFOR

\STATE {\bfseries Return:} $\mathcal{G}$

\end{algorithmic}
\end{algorithm}

\begin{table*}[t]
\caption{Main experiments are conducted by repeating the sampling process five times for each prefix, and the average score is reported. Metrics where the algorithm outperforms others are highlighted in bold. * denotes a significant difference compared to other methods ($p < 0.05$).}
\label{main-exp}
\centering
\small % or \tiny
\setlength\tabcolsep{2.5pt} % Adjust this value as needed
\begin{tabular}{llcccccccccccc}
\toprule
\multirow{2}{*}{\textbf{LM}}& \multirow{2}{*}{\textbf{Decoding}}& \multicolumn{6}{c}{\textbf{WritingPrompts}} & \multicolumn{6}{c}{\textbf{WikiText-103}}\\
\cmidrule(lr){3-8} \cmidrule(lr){9-14}
& &\textbf{rep-2} & \textbf{rep-3} & \textbf{rep-4} & \textbf{Diversity} & \textbf{MAUVE} & \textbf{Coherence}   &\textbf{rep-2} & \textbf{rep-3} & \textbf{rep-4} & \textbf{Diversity} & \textbf{MAUVE} & \textbf{Coherence}   \\
\midrule
&Human& 3.63   & 0.87   & 0.39  &  95.27    & -- & 27.12 & 5.42 & 1.20 & 0.41 & 93.18 & -- & 62.94  \\
\midrule
\multirow{6}{*}{\rotatebox{90}{GPT2-XL}} &greedy & 82.00  &80.62   & 79.66   & 1.82   & 0.61  &  \textbf{56.27$^*$} & 71.86  & 68.33  & 65.95  & 6.98   & 2.03 & 57.81  \\
&top-$k$ & 12.78 & 6.68 & 4.70 & 80.50 & 22.28 & 47.69 & 9.96 & 3.57 & 1.96 & 86.08 & 77.15 & \textbf{61.19}  \\
&mirostat &  58.48 &  51.81 & 47.32 & 18.17 & 5.61 & 50.10 & 45.38 & 36.21 & 30.74 & 31.11 & 36.06 & 60.51 \\
&top-$p$ &  12.72 & 6.86 & 4.96 & 80.24 & 21.20 & 48.01 &  8.11 & 2.74 & 1.56 & 88.82 & 78.22 & 60.00  \\
&typical & 13.28  & 7.62 & 5.83 & 79.51 & 20.36 & 46.98 & 6.73 & 2.13 & 1.19 & 90.71 & 78.49 & 59.66 \\
&$\eta$-sampling & 8.92 & 4.43 & 3.33 & 86.26 & 22.12 & 45.15 &  7.80 & 2.55 & 1.45 & 89.44 & 77.94 & 59.90  \\
&adaptive  (ours)& \textbf{6.86$^*$} & \textbf{3.66$^*$} & \textbf{3.07$^*$} & \textbf{89.30$^*$} & \textbf{25.18$^*$} &  37.71 &  \textbf{5.20$^*$} & \textbf{1.43$^*$} & \textbf{0.80$^*$} & \textbf{93.23$^*$} & \textbf{79.57} & 55.98 \\
\midrule
&Human& 7.44 & 1.59 & 0.64 & 90.67 & -- & 32.84 & 4.73 & 1.00 & 0.33 & 94.10 & -- & 60.34 \\
\midrule
\multirow{6}{*}{\rotatebox{90}{Llama2-7B}} & greedy & 77.32 & 74.67  & 72.84  & 5.58   & 1.51  & 46.40 & 43.71  & 36.70 & 32.09   & 34.54  & 22.28  & 55.85  \\
&top-$k$ & 9.13 & 4.45 & 3.27 & 87.18 & 23.57 & 23.02 & 7.18 & 2.44 & 1.20 & 90.01 & 74.80 & 58.49 \\
&top-$p$ & 8.33 & 4.10 & 3.29 & 88.38 & 23.60 & 23.05 & 7.07 & 2.49 & 1.33 & 90.05 & 77.19 & 58.19\\
&mirostat &  13.93 & 8.49 & 6.65 & 79.71 & 21.39 &  \textbf{56.23$^*$} & 12.45 & 6.14 & 3.92 & 80.63 & 69.22 & \textbf{65.79$^*$} \\
&typical & 7.16 & 3.55 & 2.91 & 90.33 & 25.78 & 22.56 & 7.06 & 2.48 & 1.27 & 90.08 & 76.04 & 57.96\\
&$\eta$-sampling & \bf 6.60$^*$ & 2.96 & 2.35 & \textbf{90.99} & 24.43 & 23.19  & \textbf{5.93$^*$} & \textbf{1.84$^*$} & \textbf{0.90$^*$} & \textbf{91.96$^*$} & \textbf{78.26} & 57.46\\
&adaptive (ours)& 7.02 & \textbf{2.63$^*$} & \textbf{1.58$^*$} & 89.98 & \textbf{35.26$^*$} & 35.10 & 6.81 & 2.31 & 1.24 & 90.55 & 76.74 & 58.22 \\
\bottomrule
\end{tabular}
\end{table*}

\section{Experiments}
\subsection{Setups}
\paragraph{Models} We assess the performance of our method on the GPT2-XL model \cite{radford2019language}, the Llama2-7B model, and its associated chat version \cite{touvron2023llama}. 

\paragraph{Datasets}
We explore two open-ended text generation applications: document continuation using the WikiText-103 dataset \cite{merity2017pointer}, which contains a large collection of Wikipedia articles. Another is story generation on the WritingPrompts dataset \cite{fan2018hierarchical}, a notably challenging endeavor. This task involves crafting imaginative continuations based on abstract, high-level story prompts contributed by online users. 

\paragraph{Baselines}
Our proposed algorithm is compared against various decoding algorithms. Top-$k$ method \cite{fan2018hierarchical} maximizes $\sum_{x \in \mathcal{A}_{\boldsymbol{x}_{<t}}} P_\theta\left(x \mid \boldsymbol{x}_{<t}\right)$ while constraining $|\mathcal{A}_{\boldsymbol{x}_{<t}}| = k$. Here, $k$ serves as the predefined hyperparameter. Top-$p$ method \cite{holtzman2019curious} minimizes $|\mathcal{A}_{\boldsymbol{x}_{<t}}|$ while constraining $\sum_{x \in \mathcal{A}_{\boldsymbol{x}_{<t}}} P_\theta\left(x \mid \boldsymbol{x}_{<t}\right) \geq p$. $p$ is the preset hyperparameter. Mirostat method \cite{basu2020mirostat} is perplexity-controlled by estimating $|\mathcal{A}_{\boldsymbol{x}_{<t}}|$ based on Zipf's law and the target cross-entropy. Typical decoding \cite{meister2023locally} tries to minimize the absolute distance $\sum_{x \in \mathcal{A}_{\boldsymbol{x}_{<t}}} \mid \mathrm{H}\left(X_t \mid \boldsymbol{X}_{<t}=\boldsymbol{x}_{<t}\right)  +\log P_\theta\left(x \mid \boldsymbol{x}_{<t}\right)|$, by constraining the $\sum_{x \in \mathcal{A}_{\boldsymbol{x}_{<t}}} P_\theta\left(x \mid \boldsymbol{x}_{<t}\right) \geq \tau$. $\tau$ is the preset hyperparameter. $\eta$-sampling \cite{hewitt2022truncation} considers the allowed set $\mathcal{A}_{\boldsymbol{x}_{<t}} = \{x\in \mathcal{V} |P_\theta(x|\boldsymbol{x}_{<t}) > \eta\}$, where $\left.\eta=\min \left(\epsilon, \alpha \exp \mathrm{H}\left(X_t \mid \boldsymbol{X}_{<t}=\boldsymbol{x}_{<t}\right)\right)\right\}$, and $\alpha$ and $\epsilon$ are hyperparameters. 

\paragraph{Automatic evaluation}
\textbf{Repetition} is defined as $\textbf{rep-n} = 1.0 - \frac{\#\text{unique n-gram}}{\#\text{total n-gram}}$ to measure sequence-level repetition according to the portion of duplicate n-grams \cite{Welleck2020Neural}. \textbf{Diversity} considers the repetition of generated text at different $n$-gram levels and can be calculated as follows: $\prod_{n=2}^4(1.0 - \frac{\textbf{rep-n}}{100})$ \cite{meister2023locally}. \textbf{MAUVE} \cite{pillutla2021mauve} is a metric that quantifies the similarity in token distribution between generated text and human-written text. Semantic \textbf{coherence} is defined as $v_{\mathcal{C}}^{\top} v_{\mathcal{G}} /\left(\left|v_{\mathcal{C}}\right| \cdot\left|v_{\mathcal{G}}\right|\right)$, where $v_{\mathcal{C}}=\operatorname{SimCSE}(\mathcal{C})$ and $v_{\mathcal{G}}=\operatorname{SimCSE}(\mathcal{G})$. It is computed as the cosine similarity between the sentence embeddings of the prefix and the continuation, represented by SimCSE \cite{gao-etal-2021-simcse}.

\subsection{Implementation Details}
We randomly select 1200 data samples from the training set of each dataset, use 1000 samples of them to evaluate different decoding algorithms and use the remaining 200 samples to select hyperparameters. Hyperparameter scans can be found in \cref{hyperparameter}, and we select the hyperparameters that result in the optimal MAUVE score \cite{meister2023locally} to guarantee fairness. In the context of WritingPrompts, we directly employ the original prompt from the dataset to feed models to execute the content continuation task. In the case of WikiText-103, a prefix with fixed-length 32 is utilized for the continuation task. The maximum generation length is constrained to 256 tokens for both datasets.  To ensure robustness, we conduct the sampling process five times for each prefix and subsequently report the average score. Furthermore, all decoding methods adopt multinomial sampling within the truncated space.

\subsection{Results}

The primary experiments are presented in \cref{main-exp}. The MAUVE score is measured by referencing human-written texts. From the perspective of human texts, there is a high diversity. Since the WritingPrompts can have many story continuations given a prompt, coherence on this dataset is low even when written by humans. It can be seen that using the greedy decoding method for LMs trained by MLE would produce a serious repetition phenomenon, leading to high coherence but low MAUVE metric and low diversity. This indicates that the generated text is in contrast to the human-written text. Conversely, this phenomenon is significantly improved when we expand the set of allowable samples for the next token. Therefore, we must combine these metrics to analyze the results in \cref{main-exp}.

When utilizing the adaptive decoding method, the two LMs generally perform better than other methods on two datasets regarding MAUVE and diversity. Besides, on the WritingPrompts dataset, we improved coherence compared with written text by humans. This observation suggests that our approach is well-suited for datasets like WritingPrompts, which excels in producing higher-quality text with limited prompts. In contrast to our approach, Mirostat determined the size of candidate sets by managing perplexity, resulting in a notable enhancement in coherence. Nevertheless, upon examining the diversity and MAUVE indicators, it appears this could be attributed to partial repetition of generation. Additionally, we found that $\eta$-sampling is quite competitive to our proposed method, and we perform a detailed comparison in \cref{analysis}.

\subsection{Human Evaluation}
We randomly select 200 samples from WikiText-103 for human evaluation, as the prefixes of this dataset offer ample contextual information. The evaluation guidance is outlined in \cref{app: human-eval}, with the results presented in \cref{human-eval}.

\begin{table}[thp]
\caption{Human evaluation of 200 WikiText-103 samples using the Llama2-7B model. * denotes a significant difference compared to other methods ($p < 0.05$).}
\label{human-eval}
\vskip 0.15in
\begin{center}
\begin{small}
\begin{tabular}{ccccc}
\toprule
\multicolumn{2}{c}{\textbf{A Is Better}} & \textbf{Neutral} & \multicolumn{2}{c}{\textbf{B Is Better}}\\
\midrule
adaptive & 4.000 & 45.000 & \textbf{51.000$^*$} & human \\
adaptive & \textbf{64.500$^*$} & 17.000 & 18.500 & top-$k$ \\
adaptive & \textbf{64.000$^*$}  & 15.000 & 21.000 & top-$p$ \\
adaptive & \textbf{50.125$^*$}  & 30.875 & 19.000 & $\eta$-sampling \\
\bottomrule
\end{tabular}
\end{small}
\end{center}
\vskip -0.1in
\end{table}
Observing \cref{human-eval}, it becomes evident that while the adaptive decoding algorithm falls short of achieving parity with human-generated content, a notable 45\% of the evaluated data poses a challenge in distinguishing between human-written and adaptively decoded text. Simultaneously, it is apparent that text decoded through adaptive decoding surpasses that generated through top-$k$, top-$p$, and $\eta$-sampling algorithms. Furthermore, during human evaluation, we found that $\eta$-sampling occasionally generates excessively short text.

\subsection{Adaptive Decoding for Policy Fine-tuned Model}
\label{policy}
Llama 2-chat undergoes fine-tuning with an objective aligned to human preferences rather than MLE, as detailed in \citet{touvron2023llama}. We evaluate the performance of adaptive decoding compared to top-$p$ sampling on the Llama2-7B-chat model. We utilized its official generation configuration with the temperature set to 0.6. As illustrated in \cref{ana2}, our decoding algorithm demonstrates the ability to enhance generation quality. This improvement is marked by increased diversity without compromising coherence. %The low coherence observed in top-$p$ may be attributed to sampling certain irrational tokens contributing to the cumulative probability within the top-$p$ algorithm. 

\begin{table}[htp]
\caption{Experiments are conducted by repeating the sampling process five times for each prefix from WritingPrompts, and the average score is reported. $p$ is set to $0.95$ while threshold for adaptive decoding is set to $0.0005$. * denotes a significant difference compared to other methods ($p < 0.05$).}
\label{ana2}
\centering
\small % or \tiny
\setlength\tabcolsep{6pt} % Adjust this value as needed
\begin{tabular}{llcccc}
\toprule
\multirow{2}{*}{\textbf{Decoding}} & \multicolumn{3}{c}{\textbf{Llama2-7B-chat}}\\
\cmidrule(lr){2-4} 
& \textbf{MAUVE} & \textbf{Diversity}  & \textbf{Coherence}  \\
\midrule
top-$p$  & 31.85  & 87.41 & 49.85 \\
adaptive  & \textbf{37.06$^*$} & \textbf{93.99$^*$} & \textbf{54.39$^*$} \\
\bottomrule
\end{tabular}
\end{table}

Besides, we demonstrate the usefulness of our method on real-world test sets such as the MT-Bench benchmark \cite{zheng2024judging}. We utilize GPT-4 \cite{openai2023gpt4} to assess the generations produced by top-$p$ and adaptive decoding methods, maintaining the same evaluation settings as \citet{zheng2024judging}. This approach demonstrates a notable agreement rate exceeding 80\% between human judges and GPT-4 evaluation. Regarding MT-bench experiments for LLama2-7B-chat testing, the results are summarized in \cref{ana3}:

\begin{table}[htp]
\caption{Evaluating text generation on the MT-Bench dataset using GPT-4 (score: $0\sim10$). $p$ is set to $0.95$ while the threshold for adaptive decoding is set to $0.0005$.}
\label{ana3}
\centering
\small % or \tiny
\setlength\tabcolsep{6pt} % Adjust this value as needed
\begin{tabular}{llcccc}
\toprule
\multirow{2}{*}{\textbf{Decoding}} & \multicolumn{3}{c}{\textbf{Llama2-7B-chat}}\\
\cmidrule(lr){2-4} 
& \textbf{turn-1} & \textbf{turn-2}  & \textbf{mean}  \\
\midrule
top-$p$  & 6.7688  & 5.7625 & 6.2656 \\
adaptive  & \textbf{6.9000} & \textbf{5.9750} & \textbf{6.4375} \\
\bottomrule
\end{tabular}
\end{table}

\subsection{Reasoning Tasks}
\label{reasoning}
We utilize Llama2-7B-chat and Llama2-13B-chat models to tackle mathematical reasoning tasks in GSM8K \cite{cobbe2021training} and MultiArith. We randomly selected 300 samples from each dataset and instructed the models with the same instruction shown in \cref{Instruction}. The results were checked by humans.

\begin{table}[htp]
\caption{Zero-shot evaluation. $p$ is set to $0.95$ while threshold for adaptive decoding is set to $0.0005$. All the models use their chat version.}
\label{ana4}
\centering
\small % or \tiny
\setlength\tabcolsep{4pt} % Adjust this value as needed
\begin{tabular}{lccccc}
\toprule
\multirow{2}{*}{\textbf{Decoding}} & \multicolumn{3}{c}{\textbf{GSM8K}}\\
\cmidrule(lr){2-4} 
& \textbf{Llama2-7B} & \textbf{Llama2-13B} & \textbf{Llama2-70B} \\
\midrule
greedy & 29.33 & 39.33 & 56.33\\
top-$p$  & 29.00  &  37.00 & 56.00\\
adaptive  & \textbf{30.00} & \textbf{40.33} & \textbf{58.67}  \\
\midrule
\multirow{2}{*}{\textbf{Decoding}} & \multicolumn{3}{c}{\textbf{MultiArith}}\\
\cmidrule(lr){2-4} 
& \textbf{Llama2-7B} & \textbf{Llama2-13B}& \textbf{Llama2-70B}  \\
\midrule
greedy &  69.33 & 66.67 &  84.67\\
top-$p$  & 68.33  & \textbf{73.33}  & 85.33 \\
adaptive  & \textbf{69.33}  & 71.33 & \textbf{87.00} \\
\bottomrule
\end{tabular}
\end{table}

Initially, we observed that greedy decoding yielded satisfactory results, as these models were fine-tuned based on human instruction, establishing our baseline.  Subsequently, we expanded the candidate size to investigate whether the models exhibited sycophancy or encountered other issues with reasoning, potentially leading to decreased performance.  It was observed that increasing the candidate size improved performance for relatively smaller models, and the effect was the opposite for larger models.  Nonetheless, our adaptive decoding strategy maintained performance levels. 
\section{Analysis}
\label{analysis}
\subsection{Ablation Study}

\paragraph{Effect of threshold}

As outlined in \cref{dconf}, $\Delta$Conf strongly correlates with the current probability and the size of the allowed set. Consequently, adjusting various thresholds empowers us to manage both of them. The results in \cref{ana11} show that the hyperparameter $\epsilon$ is readily controllable. As $\epsilon$ increases, the allowed set size $k$ diminishes, maintaining a relatively constant cumulative probability within the set while exhibiting an increase in standard deviation. Concurrently, there is an observable rise in model confidence. This trend suggests heightened confidence levels may lead to degradation issues within a constrained allowable set.

\begin{table}[t]
\caption{Effect of threshold and sampling ways. Results are obtained from WritingPrompts.}
\label{ana11}
\centering
\small % or \tiny
\setlength\tabcolsep{3pt} % Adjust this value as needed
\begin{tabular}{ccccccccccc}
\toprule
\multirow{2}{*}{\textbf{LM}}& \multirow{2}{*}{$\epsilon$} & \multicolumn{3}{c}{\textbf{Random}}\\
\cmidrule(lr){3-5} 
& & $\mathbf{k}$ & \textbf{cumulative prob.} & \textbf{Conf} \\
\midrule
\multirow{4}{*}{\rotatebox{90}{GPT2-XL}}
& 0.001 & 30.34 $\pm$ 13.23 & 77.34 $\pm$ 14.67 & 64.89 $\pm$ 14.13  \\
& 0.005 & 8.13 $\pm$ 4.17 & 73.20 $\pm$ 18.40 & 71.82 $\pm$ 14.12 \\
& 0.01& 4.31 $\pm$ 2.58 & 70.44 $\pm$ 21.36 & 75.30 $\pm$ 14.83\\
& 0.02 & 1.80 $\pm$ 1.27 & 75.28 $\pm$ 26.22 & 84.06 $\pm$ 16.30 \\
\midrule
\multirow{4}{*}{\rotatebox{90}{Llama2}} 
& 0.001 & 31.81 $\pm$ 15.23 & 79.22 $\pm$ 13.84 & 65.08 $\pm$ 14.91 \\
& 0.005 & 8.14 $\pm$ 4.67 & 77.65 $\pm$ 17.07 & 74.41 $\pm$ 14.41  \\
& 0.01  & 4.51 $\pm$ 2.70 & 73.90 $\pm$ 19.78 & 77.00 $\pm$ 14.05\\
& 0.02 & 2.13 $\pm$ 1.47 & 75.51 $\pm$ 23.15 & 83.89 $\pm$ 14.21 \\
\toprule
\multirow{2}{*}{\textbf{LM}}& \multirow{2}{*}{$\epsilon$} & \multicolumn{3}{c}{\textbf{MultiNomial}}\\
\cmidrule(lr){3-5} 
& & $\mathbf{k}$ & \textbf{cumulative prob.} & \textbf{Conf} \\ 
\midrule
\multirow{4}{*}{\rotatebox{90}{GPT2-XL}}
& 0.001 & 30.40 $\pm$ 12.99 & 77.50 $\pm$ 14.59 & 64.89 $\pm$ 13.92 \\
& 0.005 & 8.21 $\pm$ 4.17 & 73.20 $\pm$ 18.35 & 71.77 $\pm$ 14.03 \\
& 0.01 & 4.25 $\pm$ 2.57 & 70.17 $\pm$ 21.75 & 75.20 $\pm$ 15.12 \\
& 0.02 & 1.73 $\pm$ 1.23 & 76.93 $\pm$ 25.70 & 85.15 $\pm$ 15.98 \\
\midrule
\multirow{4}{*}{\rotatebox{90}{Llama2}} 
& 0.001 & 21.73  $\pm$ 15.79 & 89.09 $\pm$ 11.00 & 77.17 $\pm$ 15.12 \\
& 0.005  & 6.47 $\pm$ 4.56 & 83.26 $\pm$ 16.32 & 80.44 $\pm$ 14.53\\
& 0.01   & 3.39 $\pm$ 2.61 & 81.39 $\pm$ 19.45 & 83.76 $\pm$ 14.61 \\
& 0.02  &  1.67 $\pm$ 1.27 & 83.00 $\pm$ 22.63 & 89.00 $\pm$ 14.05\\
\bottomrule
\end{tabular}
\end{table}

% \begin{table}[t]
% \caption{Classification accuracies for Naive Bayes and Flexible Bayes on various datasets.}
% \label{ana2}
% \centering
% \small % or \tiny
% \setlength\tabcolsep{4pt} % Adjust this value as needed
% \begin{tabular}{llccccccccr}
% \toprule
% \multirow{2}{*}{\textbf{LM}}& \multirow{2}{*}{$\epsilon$} & \multicolumn{3}{c}{\textbf{Random}}\\
% \cmidrule(lr){3-5} 
% &  & \textbf{MAUVE} & \textbf{Diversity}  & \textbf{Coherence} \\
% \midrule
% \multirow{4}{*}{\rotatebox{90}{GPT2-XL}}
% & 0.001 & 31.39 & 98.26 & 45.18 \\
% & 0.005 & 28.66 & 94.69 & 53.19 \\
% & 0.01 & 30.21 & 83.50 & 57.02  \\
% & 0.02 & 21.80 & 32.24 & 60.33 \\
% \midrule
% \multirow{4}{*}{\rotatebox{90}{Llama2}}
% & 0.001 & 12.19 & 98.38 & 25.67 \\
% & 0.005 & 43.21 & 95.29 & 34.52  \\
% & 0.01  & 53.58 & 88.62 & 42.55\\
% & 0.02 & 52.70 & 50.62 & 46.57 \\
% \toprule
% \multirow{2}{*}{\textbf{LM}}& \multirow{2}{*}{$\epsilon$} & \multicolumn{3}{c}{\textbf{MultiNomial}}\\
% \cmidrule(lr){3-5} 
% & & \textbf{MAUVE} & \textbf{Diversity}  & \textbf{Coherence}   \\
% \midrule
% \multirow{4}{*}{\rotatebox{90}{GPT2-XL}}
% & 0.001 & 30.77 & 98.84 & 45.53 \\
% & 0.005 & 19.09 & 95.06 & 54.64 \\
% & 0.01  & 47.77 & 80.47 & 58.65 \\
% & 0.02  & 12.24 & 29.37 & 61.40 \\
% \midrule
% \multirow{4}{*}{\rotatebox{90}{Llama2}}
% & 0.001 & 54.76 & 84.80 & 33.37\\
% & 0.005 & 50.98 & 67.06 & 44.10\\
% & 0.01 &  42.25 & 50.14 & 43.81\\
% & 0.02 &  11.63 & 18.31 & 44.08\\
% \bottomrule
% \end{tabular}
% \end{table}

\paragraph{Effect of sampling ways}
Beyond the significance of truncation space, the method employed for sampling within this space plays a crucial role. In our investigation, we compare random sampling with multinomial sampling. The latter method preserves the probability distribution of the allowed set and samples tokens according to that distribution. As illustrated in \cref{ana11}, when considering the GPT2-XL model, different sampling methods exhibit only marginal impacts on the results. This can be attributed to the observation that, for the smaller LM, the distribution within the top-$k$ space tends to flatten, as depicted in \cref{fig:long_tailed}.

However, in the case of the Llama2-7B model, an interesting observation emerges: as $\epsilon$ increases, the multinomial method enhances model confidence and converges on a restricted allowed set. This phenomenon may arise from multinomial sampling being driven by the probability distribution, leading the model to favor tokens with the highest probabilities. Consequently, this contributes to degradation issues for models trained with MLE. Conversely, based on the outcomes of random sampling, including a few low-probability tokens appears to mitigate this degradation phenomenon. This suggests we need to lower $\epsilon$ when using multinomial sampling.

\begin{figure*}
    \centering
    \includegraphics[width=17cm]{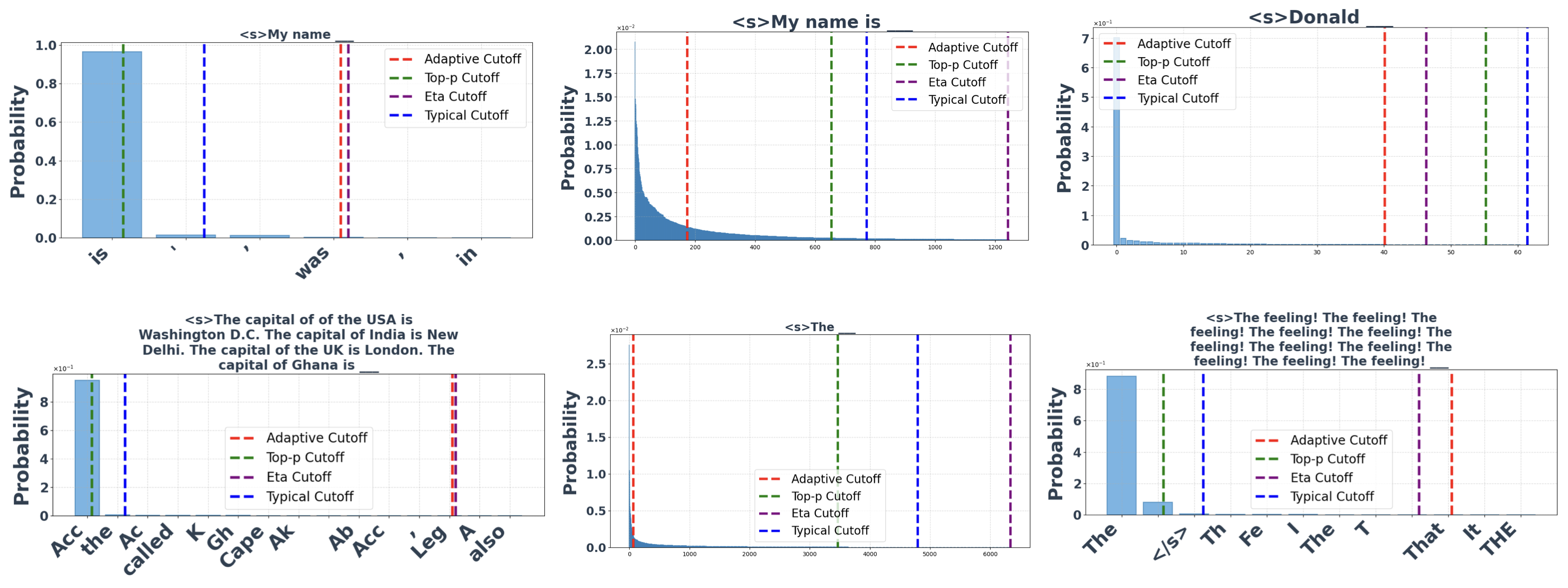}
    \caption{Unit tests of the truncation behavior of top-$p$, typical, $\eta$-sampling and adaptive decoding on CheckList-inspired prefixes. $p=0.90$, typical$=0.90$, $\eta=\sqrt{0.0009}=0.03$, and threshold for adaptive decoding $\epsilon$ is $0.0005$.}
    \label{fig:enter-label}
\end{figure*}

\subsection{Deocding Latency}

We employed these decoding algorithms to generate text based on 100 distinct prefixes. Subsequently, we computed the average token decoding time (after deriving logit from the model). Our implementation relies on the top-$p$ algorithm. As observed from \cref{ana22}, there is a clear trend of decreasing average token decoding time with a reduced vocabulary size. 

\begin{table}[htp]
\caption{Decoding latency. The computational hardware is NVIDIA RTX 3090, with the model loaded in float16. The unit of measurement is milliseconds (ms). } %The vocabulary size for GPT2-XL and Llama2-7B is 50257 and 32000, respectively. }
\label{ana22}
\centering
\small % or \tiny
\setlength\tabcolsep{5pt} % Adjust this value as needed
\begin{tabular}{llcccc}
\toprule
\textbf{Decoding} & \textbf{GPT2-XL} & \textbf{Llama2-7B}\\
\midrule
greedy &  0.0221 & 0.0265 \\
top-$k$ & 0.3116 & 0.3466 \\
top-$p$  & 0.4085   & 0.3706  \\
adaptive  &  0.5741 & 0.5265 \\
\bottomrule
\end{tabular}
\end{table}

\subsection{Qualitative Analysis}
We compare the truncation behaviors of different decoding algorithms through examples extracted from \citet{ribeiro-etal-2020-beyond}.

In our analysis of the top-$p$ algorithm, we observed that both top-$p$ and typical decoding exhibit similar truncation behaviors.  They effectively truncate low-entropy distributions but are less suitable for high-entropy distributions.  Adaptive decoding shows behavior similar to $\eta$-sampling for low-entropy distributions, offering enhanced diversity compared to the previous two methods.  However, our method uniquely establishes a bounded candidate set for high-entropy distributions, setting it apart from the other three approaches.
\section{Related Work}

\paragraph{Deterministic algorithms}
Establishing a reasonable size for the next candidate token set seems challenging. Greedy decoding directly uses the highest probability as the next token. Beam search maintains multiple probable tokens as the candidates of the next tokens. However, they often lead to high repetition. Contrastive search \cite{su2022contrastive} predefines the top-$k$ space and selects the discriminative token concerning the previous context. However, it is time-consuming and may select the irrational token. Penalty decoding aims to improve greedy decoding by adding penalties to tokens that have appeared before since the model tends to increase the probability of these tokens \cite{zhu-etal-2023-penalty}. Look-back decoding \cite{xu-etal-2023-look} adds randomness based on greedy decoding, which uses KL divergence to track the distribution distance between the current and historical decoding steps and limit the next token probability distribution to a reasonable distance from the history.

\paragraph{Stochastic algorithms}
The most typical random decoding algorithms are top-$k$ \cite{fan2018hierarchical} and top-$p$ \cite{holtzman2019curious}algorithms. Some later works are proposed to improve the truncation space. Mirostat \cite{basu2020mirostat} also proposes an adaptive decoding algorithm. In contrast to our work, they aim to control the perplexity close to a target value. They first calculate one value assuming words follow Zipf's law and then design a function with this value to calculate the size of the allowed set. $\eta$-sampling \cite{hewitt2022truncation} proposes to truncate words below an entropy-dependent probability threshold. Typical decoding \cite{meister2023locally} measures the distance between negative log probabilities and conditional entropy and then takes words from this list to cover $p$ percent of the distribution. Recently, \citet{finlayson2024closing} proposed basis-aware-threshold (BAT) sampling that uses the softmax matrix to find tokens that might have non-zero true probability without relying on a threshold.

\section{Conclusion}
In conclusion, our research introduces a novel decoding algorithm named \textit{adaptive decoding} for open-ended text generation, leveraging the thought of entropy reduction to dynamically determine the size of a suitable candidate set. We demonstrated its effectiveness in balancing coherence and diversity through experiments, outperforming existing decoding methods across different language model scales.

\section*{Impact Statement}
This paper presents work that aims to advance the field of natural language generation. There are many potential societal consequences of our work, none of which we feel must be specifically highlighted here.

% \section*{Software and Data}
% \label{software}

% A quick estimation of the size of the allowed set $k$ is implemented by Pytorch as follows:

% \begin{lstlisting}[language=Python, caption=Your PyTorch code, label=pytorch_code]
% import torch
% import torch.nn as nn

% # estimation of the size of the allowed set. 
% def adaptive_k(logit):
%     '''
%     logit is the score predicted by the LM
%     V is the size of the vocabulary 
%     epsilon is the given threshold
%     '''
%     x = nn.functional.softmax(logit, dim=-1)
%     sorted_values, sorted_indices = torch.sort(x, descending=True)
%     prob = sorted_values
%     up_bound = torch.log2(V)
%     index = torch.arange(2, V)
%     cumulative_sum = torch.cumsum(prob, dim=0)
    
%     A = prob * torch.log(prob * (V - tensor) / (1 - cumulative_sum))
%     B = torch.log((1 - cumulative_sum) / (V - tensor) / ((1 - cumulative_sum + prob) / (V + 1 - tensor)))
%     delta_conf = (A + (1 - cumulative_sum + prob) * B) / up_bound
%     indices = torch.where(delta_conf > epsilon)
%     k = indices[0][-1].item()
%     return k
% \end{lstlisting}

\section*{Acknowledgements}
This paper is supported by the National Natural Science Foundation of China (62176153) and the Shanghai Municipal Science and Technology Major Project (2021SHZDZX0102, as the MoE Key Lab of Artificial Intelligence, AI Institute, Shanghai Jiao Tong University).

% \section*{Acknowledgements}

% \textbf{Do not} include acknowledgements in the initial version of
% the paper submitted for blind review.

% If a paper is accepted, the final camera-ready version can (and
% probably should) include acknowledgements. In this case, please
% place such acknowledgements in an unnumbered section at the
% end of the paper. Typically, this will include thanks to reviewers
% who gave useful comments, to colleagues who contributed to the ideas,
% and to funding agencies and corporate sponsors that provided financial


\bibliography{example_paper}
\bibliographystyle{icml2024}

%%%%%%%%%%%%%%%%%%%%%%%%%%%%%%%%%%%%%%%%%%%%%%%%%%%%%%%%%%%%%%%%%%%%%%%%%%%%%%%
%%%%%%%%%%%%%%%%%%%%%%%%%%%%%%%%%%%%%%%%%%%%%%%%%%%%%%%%%%%%%%%%%%%%%%%%%%%%%%%
% APPENDIX
%%%%%%%%%%%%%%%%%%%%%%%%%%%%%%%%%%%%%%%%%%%%%%%%%%%%%%%%%%%%%%%%%%%%%%%%%%%%%%%
%%%%%%%%%%%%%%%%%%%%%%%%%%%%%%%%%%%%%%%%%%%%%%%%%%%%%%%%%%%%%%%%%%%%%%%%%%%%%%%
\newpage
\appendix
\onecolumn
\section{Limitations}
We acknowledge that implementing the adaptive decoding algorithm introduces additional overhead. We hope that future work can address and eliminate this overhead. While we have theoretically demonstrated the rationale behind our method, further experimental validation through more extensive reasoning tasks is needed to confirm its effectiveness. Additionally, we anticipate the development of improved metrics for evaluating the results of model generation. Moreover, our verification of the decoding algorithm has been limited to the language models only and has not yet been extended to other areas like the multimodal domain.

\section{Entropy Change}
\label{entropy}
It is evident from \cref{fig: entropy} that the entropy of the model experiences notable fluctuations during text generation. Furthermore, distinct word segmentation among various models yields varying upper bounds for text generation entropy in each model. The amplitude of entropy changes between models is also influenced by the quality of model pre-training. Therefore, normalizing entropy becomes important to mitigate the impact of dimensional effects.

\begin{figure*}[htp]
  \centering

  \subfigure[The value of entropy at each generation step of the GPT2-XL model guided by human-written text. ]{
    \includegraphics[width=7.4cm]{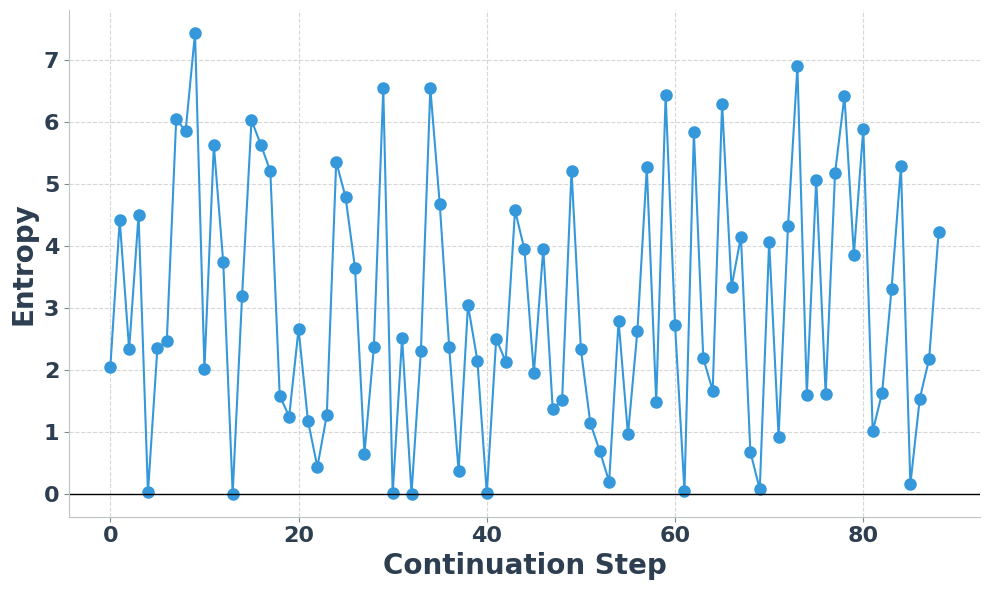}
    
  }
  \subfigure[The value of entropy at each generation step of the Llama2-7B model guided by human-written text. ]{
    \includegraphics[width=7.4cm]{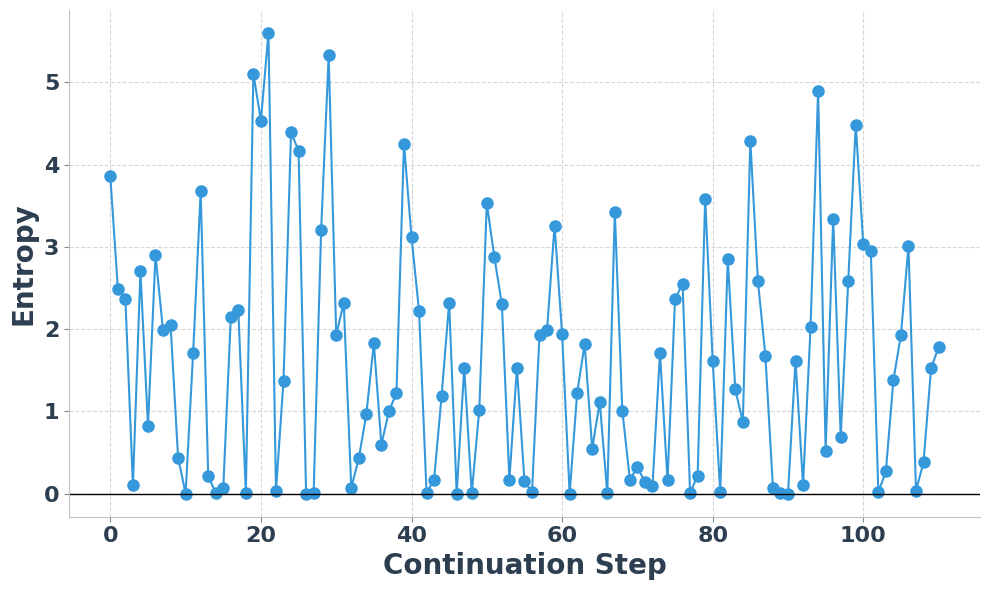}
    
  }
  
  \caption{Human written text: ``Barack Obama was born in 1961. He was raised in Hawaii by his mother and grandparents. Growing up, Obama faced the challenges of being biracial, with a Kenyan father and an American mother. Despite these challenges, he excelled academically and eventually attended Columbia University in New York City.'' %After completing his education, Obama worked as a community organizer in Chicago, focusing on issues such as job training and housing for low-income communities. His passion for social justice and community development propelled him into the world of politics.
  }
  \label{fig: entropy}
\end{figure*}

\section{Proof}
\label{proof}
\subsection{Proof of \cref{increase}}

Examine \cref{eq: Deltaconf} closely, and it becomes evident that $p_k\log p_k$ represents the average uncertainty carried into the known segment. On the other hand, the latter term $(1 - \sum_i^k p_i) \log\frac{(1 - \sum_i^k p_i)}{|\mathcal{V}| - k} - (1 - \sum_i^{k-1} p_i) \log\frac{(1 - \sum_i^{k-1} p_i)}{|\mathcal{V}| - k + 1}$ reflects the certainty introduced into the unknown portion. It can be shown that the increase in certainty surpasses the introduction of uncertainty as the known part expands.

\begin{corollary}
\label{eq:cor}
When sorting the state distribution $X$ predicted by the language model in descending order after applying the softmax function, it follows that $p_k \ge \frac{1- \sum_{i}^{k-1}p_i}{|\mathcal{V}|-k + 1}$.
\end{corollary}

\begin{proof}
We acknowledge that $p_1 >  p_2 >  ...  >  p_{|\mathcal{V}|}$ with $\sum_i^{|\mathcal{V}|} p_i = 1$.  If, however, $p_k <  \frac{1- \sum_{i}^{k-1}p_i}{|\mathcal{V}|-k + 1}$, it implies the existence of $j$ within the range $[k+1, |\mathcal{V}|]$ such that $p_j >  p_k$.  This assertion results in a contradiction.
\end{proof}

\begin{corollary}
\label{eq:ineq1}
Given that the sum of probabilities for all tokens in the vocabulary is expressed as $\sum_{i}^{|\mathcal{V}|} p_i = 1$, we can deduce that $ 1 - \sum_{i}^{k-1}p_i -p_k = \sum_{k + 1}^{|\mathcal{V}|} p_i \ge 0$. Observing this inequality, it is evident that $1 - \sum_{i}^{k-1}p_i \ge p_k$.
\end{corollary}

\begin{corollary}
\label{eq: increasing}
The function Conf$_k(X)$ is monotonically increasing, as it is always guaranteed that $\Delta$Conf $\ge 0$.
\end{corollary}

\subsection{Proof of \cref{approximation}}

\begin{proof}
It is intuitively correct that $\mathrm{H}(X) = -\sum_i^{\mathcal{V}}p_i \log p_i$ converges to $-\sum_i^k p_i \log p_i$ under the condition $P_\theta(X>k + m|X>k) \rightarrow 1$. In accordance with \cref{eq:cor}, the following inequality holds:
$$
\begin{aligned}
    &\log|\mathcal{V}|\Delta \text{Conf} \ge p_k\log \frac{p_k (|\mathcal{V}| - k)}{ 1- \sum_i^{k} p_i}\\& +(1 - \sum_i^{k-1} p_i)\log\frac{1}{p_k}    + (1 - \sum_i^{k-1} p_i) \log\frac{(1 - \sum_i^{k} p_i)}{|\mathcal{V}| - k}.
\end{aligned}
$$
Moreover, based on \cref{eq:ineq1}, we can further derive:
$$
\begin{aligned}
&\log|\mathcal{V}|\Delta \text{Conf} \ge p_k\log \frac{p_k (|\mathcal{V}| - k)}{ 1- \sum_i^{k} p_i} \\ & + p_k \log\frac{(1 - \sum_i^{k} p_i)}{|\mathcal{V}| - k}  +  p_k\log\frac{1}{p_k}  \ge p_k \log 1 = 0.
\end{aligned}
$$
\end{proof}

\label{app: truncation}
Detailed mathematical proof in \cref{cor: truncation} is as follows.

Consider $$H(X)= -\sum_{i=1}^kP_\theta(x_i)\log P_\theta(x_i) - \sum_{k+1}^{|\mathcal{V}|}P_\theta(x_i)\log P_\theta(x_i).$$ 

As $P_\theta(X>k + m|X>k) \rightarrow 1$, 
\begin{equation}
\label{ddd}
    - \sum_{k+1}^{|\mathcal{V}|}P_\theta(x_i)\log P_\theta(x_i) \rightarrow 0.
\end{equation}

We define the binary random variable $$B(u)=\left\{\begin{array}{l}
1 \text { if } X \leq u \\
0 \text { if } X>u.
\end{array} \right.$$. 

Then, we can express $H(X)$ in terms of $B(u)$ as:
$$
\begin{aligned}
H(X) & = H(X \mid B(u))+I(X ; B(u)) \\
& =H(X \mid B(u)=1) \operatorname{Pr}\{B(u)=1\}+H(X \mid B(u)=0)  \times \operatorname{Pr}\{B(u)=0\}+I(X ; B(u)) \\
& =H(X(u)) \operatorname{Pr}\{B(u)=1\}+H(X \mid B(u)=0) \times \operatorname{Pr}\{B(u)=0\}+I(X ; B(u)) .
\end{aligned}
$$

As $u \rightarrow k$, we have $\operatorname{Pr}(B(u) = 1) \rightarrow 1$, leading to $H(B(u)) \rightarrow 0$. Consequently, $I(X;B(u)) \rightarrow 0$ due to $I(X;B(u)) \le H(B(u))$.

Additionally, $H(X|B(u)=0)\operatorname{Pr}(B(u)=0)$ can be simplified as:

$$
\begin{aligned}
H( & X \mid B(u)=0) \operatorname{Pr}\{B(u)=0\}\\
& =  -\sum_{i=k+1}^{\infty} P_\theta(x_i) \log \frac{P_\theta(x_i)}{\operatorname{Pr}\{B(u)=0\}} \\
= & -\sum_{i=k+1}^{\infty} P_\theta(x_i)\log P_\theta(x_i) -\log \operatorname{Pr}\{B(u)=0\}) \\
= & -\sum_{i=k+1}^{\infty}P_\theta(x_i) \log P_\theta(x_i) +\left(\sum_{i=k+1}^{\infty} P_\theta(x_i)\right) \log \operatorname{Pr}\{B(u)=0\} \\
= & -\sum_{i=k+1}^{\infty} P_\theta(x_i) \log P_\theta(x_i) +\operatorname{Pr}\{B(u)=0\} \log \operatorname{Pr}\{B(u)=0\} .
\end{aligned}
$$

As $u \rightarrow k$, we have $\operatorname{Pr}(B(u)=0) \rightarrow 0$ and \cref{ddd}, then both terms lead to zero. In conclusion, we obtain $H(X)  = H(X(k)) + \epsilon$, where $\epsilon$ is a smaller number.

\section{Hyperparameter}
\label{hyperparameter}

Choosing hyperparameters seems a challenging task, and the process of our parameter selection is illustrated in \cref{hyper-gpt} and \cref{hyper-llama}. Here are some interesting things we found:
\begin{itemize}
    \item The top-$k$ and top-$p$ algorithms enhance diversity by expanding the truncation space, leading to higher MAUVE values. This occurrence can be attributed to the ability of LMs to sample certain low-probability tokens, which are proximate to the probabilities preferred by humans. This proximity results in a closer match between the two curves illustrated in \cref{fig:generation_1}. However, it is crucial to recognize that the concurrent increase in diversity corresponds to a decrease in coherence. 
    \item Mirostat excels in effectively managing text coherence by controlling model perplexity.   The ease of controlling parameters further enhances its utility, even though some of the observed increase in coherence may be attributed to the rise in indices resulting from repetition. Additionally, our investigation reveals that Mirostat, when applied to the Llama2-7B model, maintains high coherence and ensures a certain level of diversity.
\end{itemize}

\begin{table*}[t]
\caption{Optimal hyperparameter selection for various decoding algorithms in the GPT2-XL model. The line where the bold is located designates the optimal parameter value. We repeat the sampling process five times for each prefix and report the average score.}
\label{hyper-gpt}
\centering
\small % or \tiny
\setlength\tabcolsep{3pt} % Adjust this value as needed
\begin{tabular}{lcccccccccccc}
\toprule
\multirow{2}{*}{\textbf{Decoding}}& \multicolumn{6}{c}{\textbf{WritingPrompts}} & \multicolumn{6}{c}{\textbf{WikiText-103}}  \\
\cmidrule(lr){2-7} \cmidrule(lr){8-13}
&\textbf{rep-2} & \textbf{rep-3} & \textbf{rep-4} & \textbf{Diversity} & \textbf{MAUVE} & \textbf{Coherence} &\textbf{rep-2} & \textbf{rep-3} & \textbf{rep-4} & \textbf{Diversity} & \textbf{MAUVE} & \textbf{Coherence} \\
\midrule
Human& 7.70  & 1.82 & 0.79  & 90.06  & -- & 29.57  & 5.14 & 1.02 & 0.33 & 93.66 & -- & 63.49\\
\midrule
$k=5$  & 32.86 & 24.03 & 19.55 & 49.37 & 23.94 & 53.16  & 19.18 & 10.37 & 6.90 & 70.40 & 82.29 & 62.57\\
$k=10$  &19.60 & 11.83 & 8.71 & 69.82 & 33.65 & 50.04  & 11.50 & 4.44 & 2.44 & 83.41 & 82.49 & 62.25\\
$k=15$ & 14.91 & 8.35 & 6.16 & 76.86 & 44.92 & 49.51 & 9.90 & 3.64 & 2.01 & 86.24 & \textbf{88.75} & 61.41\\
$k=20$ & 12.23 & 6.26 & 4.35 & 81.22 & \textbf{48.25} & 47.87 & 8.52 & 2.78 & 1.48 & 88.51 & 88.67 & 61.61\\
\midrule
$p=0.80$ & 25.39 & 18.79 & 15.90 & 61.61 & 30.29 & 49.13  & 12.98 & 5.94 & 3.84 & 80.89 & 77.74 & 61.11\\
$p=0.85$ & 19.78 & 13.14 & 10.80 & 70.02 & 35.10 & 48.20  & 11.23 & 4.87 & 3.02 & 83.73 & 79.61 & 61.46 \\
$p=0.90$ & 15.16 & 9.14 & 7.07 & 76.96 & 38.52 & 47.15 & 9.57 & 3.67 & 2.17 & 86.62 & 84.46 & 61.13 \\
$p=0.95$ & 12.21 & 7.04 & 5.53 & 81.43 & \textbf{51.69} & 47.35  & 7.70 & 2.70 & 1.49 & 89.24 & \textbf{88.50} & 60.87\\
\midrule
mirostat$-2.0$ &  59.05 & 52.33& 47.75 &17.26 & 13.75& 50.21  & 45.09 & 35.99 & 30.57 & 31.43 & \textbf{56.59} & 60.34 \\
mirostat$-3.0$  & 58.73 & 51.91 & 47.27 & 17.82 & 16.52 & 49.63 & 45.13 & 36.09 & 30.61 & 31.27 & 55.29 & 60.73\\
mirostat$-4.0$  & 58.62  &51.84 & 47.17& 17.69 & 16.30 & 49.28 & 45.14 & 36.07 & 30.60 & 31.11 & 56.56 & 61.01  \\
mirostat$-5.0$  &58.55 & 51.70 & 47.04 & 18.17 & \textbf{16.91} & 49.71 & 45.46 & 36.46 & 31.03 & 31.11 & 53.55 & 61.09 \\
\midrule
$\tau=0.90$ & 15.24 & 9.24 & 7.25 & 76.74 & 38.10 & 46.64  & 9.38 & 3.60 & 2.14 & 86.94 & 85.35 & 61.27 \\
$\tau=0.92$ & 14.02 & 8.09 & 6.27 & 78.65 & \textbf{43.92} & 46.63 & 9.08 & 3.35 & 2.06 & 87.47 & 84.99 & 61.00  \\
$\tau=0.95$ & 12.05 & 6.96 & 5.32 & 81.47 & 38.01 & 46.99 & 8.02 & 2.73 & 1.52 & 89.10 & 85.39 & 60.89 \\
$\tau=0.99$ & 9.22 & 4.63 & 3.64 & 85.42 & 40.57 & 46.59  & 6.74 & 1.93 & 0.95 & 91.13 & \textbf{90.07} & 59.93\\
\midrule
$\eta=0.004$ & 12.33 & 7.08 & 5.66 & 80.95 & 41.86 & 46.69  & 7.83 & 2.63 & 1.48 & 89.22 & \textbf{90.40} & 61.10  \\
$\eta=0.002$ &9.61 & 4.98 & 3.63 & 85.42 & 47.24 & 45.60 & 7.58 & 2.75 & 1.80 & 89.81 & 82.18 & 59.55\\
$\eta=0.0009$ & 9.10 & 4.58 & 3.41 & 86.58 & \textbf{59.34} & 45.67 & 6.45 & 1.75 & 0.98 & 91.38 & 87.17 & 60.33  \\
$\eta=0.0006$ & 8.99 & 4.38 & 3.67 & 85.98 & 46.23 & 44.38 & 6.44 & 1.72 & 0.82 & 91.64 & 86.45 & 60.37 \\
\midrule
$\epsilon=0.0005$ & 5.46 & 3.16 & 2.75 & 91.38 & 47.19 & 36.38 & 4.82 & 1.55 & 1.00 & 93.81 & 83.72 & 56.58\\
$\epsilon=0.001$ & 7.07 & 3.85 & 3.11 &  89.32 & \textbf{55.19} & 37.72 &  5.22 & 1.51 & 0.93  & 93.00 & \textbf{84.54} & 56.39\\
$\epsilon=0.005$ & 40.69 & 33.57 & 29.95 & 39.80 & 11.02 & 47.51 &  21.10 & 12.50 & 9.10 & 67.63 & 69.27 & 61.30  \\
$\epsilon=0.01$ & 71.14 & 67.75 & 65.28 & 9.81 & 2.89 & 50.71 & 45.17 & 37.82 & 33.96 & 34.02 & 23.08 & 61.04 \\
\bottomrule
\end{tabular}
\end{table*}

\begin{table*}[t]
\caption{Optimal hyperparameter selection for various decoding algorithms in the Llama2-7B model. The line where the bold is located designates the optimal parameter value. We repeat the sampling process five times for each prefix and report the average score.}
\label{hyper-llama}
\centering
\small % or \tiny
\setlength\tabcolsep{4pt} % Adjust this value as needed
\begin{tabular}{lcccccccccccc}
\toprule
\multirow{2}{*}{\textbf{Decoding}}& \multicolumn{6}{c}{\textbf{WritingPrompts}} & \multicolumn{6}{c}{\textbf{WikiText-103}}  \\
\cmidrule(lr){2-7} \cmidrule(lr){8-13}
&\textbf{rep-2} & \textbf{rep-3} & \textbf{rep-4} & \textbf{Diversity} & \textbf{MAUVE} & \textbf{Coherence} &\textbf{rep-2} & \textbf{rep-3} & \textbf{rep-4} & \textbf{Diversity} & \textbf{MAUVE} & \textbf{Coherence} \\
\midrule
Human& 7.70  & 1.82 & 0.79  & 90.06  & -- & 29.57  & 5.14 & 1.02 & 0.33 & 93.66 & -- & 63.49\\
\midrule
$k=5$ & 22.69 & 14.76 & 11.10 & 66.96 & 39.66 & 27.15 & 12.14 & 5.56 & 3.20 & 81.69 &  63.82 & 59.43\\
$k=10$ & 14.21 & 7.26 & 5.05 & 79.94 & 45.34 & 27.44 & 9.35 & 3.70 & 2.07 & 86.51 & 74.15 & 59.51\\
$k=15$  & 11.00 & 5.42 & 3.65 & 84.46 & 40.40 & 25.98 & 7.88 & 2.76 & 1.40 & 89.11 & 82.88 & 59.38 \\
$k=20$ & 9.62 & 4.19 & 2.94 & 86.29 & \textbf{47.36} & 26.25 & 7.44 & 2.55 & 1.40 & 89.54 & \textbf{87.07} & 59.27 \\
\midrule
$p=0.80$ & 17.85 & 10.92 & 8.23 & 74.17 & 46.40 & 27.23 & 11.22 & 5.05 & 2.94 & 83.40 & 75.67 & 59.33 \\
$p=0.85$ & 13.51 & 7.66 & 5.58 & 80.26 & 44.91 & 27.55 & 9.64 & 4.01 & 2.29 & 85.84 & 79.15 & 58.94\\
$p=0.90$ & 11.02 & 5.41 & 3.91 & 83.58 & 55.53 & 31.11 & 8.20 & 3.02 & 1.58 & 88.37 & 80.49 & 59.54\\
$p=0.95$ & 8.39 & 4.32 & 3.26 & 87.96 & \textbf{58.75} & 26.99 & 7.13 & 2.52 & 1.32 & 90.12 & \textbf{82.22} & 58.89\\
\midrule
mirostat$-2.0$ & 13.58 & 7.99 & 6.04 & 79.94 & 44.55 & 55.91 & 12.25 & 6.10 & 3.94 & 80.99 & 81.55 & 66.51 \\
mirostat$-3.0$  & 13.46 & 7.85 & 5.88 & 80.04 & \textbf{50.51} & 54.83 & 12.36 & 6.14 & 3.99 & 80.72 & \textbf{82.06} & 65.87 \\
mirostat$-4.0$ & 14.95 & 8.59 & 6.85 & 79.11 & 42.97 & 55.10 & 12.70 & 6.50 & 4.32 & 80.06 & 81.09 & 66.31   \\
mirostat$-5.0$  & 15.08 & 9.82 & 7.88 & 77.82 & 45.22 & 54.84 & 12.49 & 6.23 & 4.06 & 80.53 & 78.36 & 66.46 \\
\midrule
$\tau=0.90$ & 10.83 & 5.36 & 3.84 & 84.65 & 38.07 & 25.23 & 8.24 & 3.24 & 1.77 & 88.10 & 77.49 & 59.23\\
$\tau=0.92$ & 9.37 & 4.37 & 3.38 & 86.86 & 33.88 & 27.80 & 7.68 & 2.80 & 1.46 & 89.17 & 82.78 & 58.79 \\
$\tau=0.95$ & 9.54 & 5.03 & 3.78 & 87.09 & 40.18 & 26.80 & 6.81 & 2.20 & 1.06 & 90.61 & \textbf{87.75} & 58.83 \\
$\tau=0.99$ & 8.67 & 4.53 & 4.13 & 88.93 & \textbf{53.26} & 25.76 & 6.00 & 1.98 & 1.00 & 91.75 & 80.68 & 58.50\\
\midrule
$\eta=0.004$ & 7.86 & 3.42 & 2.50 & 88.69 & 48.93 & 26.30 & 7.15 & 2.47 & 1.30 & 90.04 & 77.47 & 58.76\\
$\eta=0.002$ & 7.17 & 3.58 & 2.62 & 89.76 & 48.34 & 26.88 & 6.41 & 2.21 & 1.17 & 91.15 & 81.49 & 58.76\\
$\eta=0.0009$ &  7.10 & 3.04 & 2.14 & 90.42 & 61.30 & 26.78 & 5.68 & 1.72 & 0.85 & 92.41 & 86.14 & 58.76 \\
$\eta=0.0006$ & 6.69 & 2.93 & 2.42 & 90.43 & \textbf{63.80} & 26.35 & 5.87 & 1.73 & 0.76 & 92.04 & \textbf{88.28} & 58.75\\
\midrule
$\epsilon=0.0005$ & 7.21 & 2.84 & 1.74 & 89.54 & \textbf{69.00} & 35.40 & 5.90 & 1.78 & 0.79 & 92.00 &  87.18 & 58.78 \\
$\epsilon=0.001$ & 8.38 & 3.54 & 2.14 & 87.76 & 59.34 & 35.92 & 6.74 & 2.16 & 1.03 & 90.65 & \textbf{87.81} & 59.26 \\
$\epsilon=0.005$ & 19.22 & 11.44 & 8.43 & 70.75 & 58.23 & 42.35 & 10.42 & 4.61 & 2.69 & 84.76 & 76.91 & 60.03\\
$\epsilon=0.01$ & 35.60 & 27.20 & 22.96 & 47.12 & 43.69 & 44.02 & 14.11 & 7.34 & 4.65 & 78.55 & 68.76 & 58.69\\
\bottomrule
\end{tabular}
\end{table*}

\section{Instrcution}
\label{Instruction}
In \cref{policy}, we use the following instruction, where the sentence is derived from the prompt in WritingPompts.
\begin{verbatim}
<s>[INST] <<SYS>>
You're a writer. You will write a story given my prefix. 
<</SYS>>

Prefix: {sentence} [/INST]
\end{verbatim}

In \cref{reasoning}, we use the following instruction, where the problem is derived from the dataset GSM8K and MultiArith.
\begin{verbatim}
<s>[INST] <<SYS>>You are a help assistant and a math expert. 
Please solve the following question and directly return me the answer.<</SYS>>
    Problem: {sentence} 
    Let's think step by step\n[/INST]
\end{verbatim}

\section{Human Evaluation}
\label{app: human-eval}
\begin{figure*}[htp]
    \centering
    \label{fig: human_eval}
    \includegraphics[width=15cm]{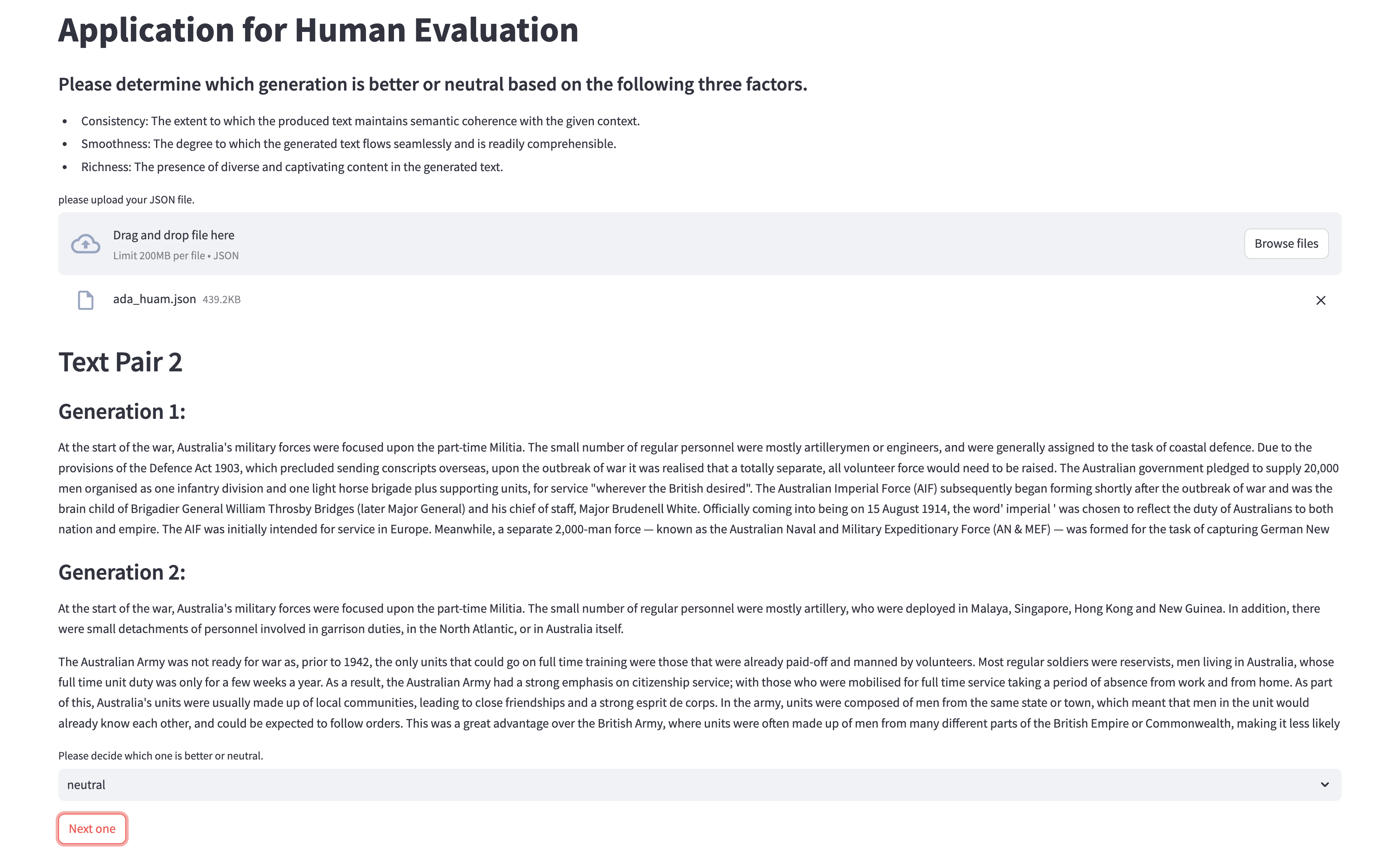}
    \caption{Our application for human evaluation. }
\end{figure*}
The human evaluation adheres to the principles of blind review, ensuring that the evaluator remains unaware of the decoding algorithms employed in the two generations or the human writing in advance. The evaluator determines the superior outcome based on three guiding principles. Our application for human evaluation is illustrated in \cref{fig: human_eval}.

\begin{itemize}
    \item Consistency: The extent to which the produced text maintains semantic coherence with the given context.
    \item Smoothness: The degree to which the generated text flows seamlessly and is readily comprehensible.
    \item Richness: The presence of diverse and captivating content in the generated text
\end{itemize}

\section{Case Study}
\label{case_study}
Given the prefix sampled from WikiText-103, which provides relevant enough context, we give the truncated human continuation and the generation of our proposed adaptive decoding. Cases can be found in \cref{case_gpt} and \cref{case_llama}.

\begin{table*}[t]
\caption{Case study: Text continuation with the GPT2-XL model. Entities are in blue fonts, logical relationships are in red, and coherent contents are highlighted in green.}
\label{case_gpt}
\vskip 0.15in
\begin{center}
\begin{small}
\begin{sc}
\begin{tabular}{p{16cm}}
\toprule
\textnormal{\textbf{\textit{Prefix}}: Following a period of rest in England, \textcolor{blue}{Little} turned down a desk assignment and volunteered to return to action on the Western Front, joining Lieutenant Colonel Collishaw's}\\\midrule
\textnormal{\textbf{\textit{Human Continuation}}:  No. 3 Squadron RNAS \textcolor{red}{in March 1918}. The unit evolved into No. 203 Squadron of the new Royal Air Force on 1 April, formed after the merger of the RNAS and the Royal Flying Corps. Now ranked captain, and again flying Sopwith Camels, \textcolor{blue}{Little} gained a further nine successes, beginning with a Fokker Triplane on 1 April, and concluding with two kills in one day \textcolor{red}{on 22 May}, an Albatros and a DFW. During this stretch of victories, \textcolor{red}{on 21 April 1918}, he was brought down unharmed by Friedrich Ehmann. \textcolor{red}{On 27 May}, \textcolor{blue}{Little} received reports of German Gotha bombers in the vicinity, and took off on a moonlit evening to intercept the raiders. As he closed with one of the bombers, his plane was caught in a searchlight beam and he was struck by a bullet that passed through both his thighs. He crash-landed in a field near Nœux, and bled to death before he was discovered the following morning by a passing gendarme. \textcolor{blue}{Little}'s skull and ankle had also been fractured in the impact; his body was identified by his friend and fellow ace, Charles Dawson Booker. 
}\\
\midrule
\textnormal{\textbf{\textit{Adaptive decoding}}:  49th Rifles. On his first day back, he led a cavalry charge which was repulsed in the village of Culmore Wood \textcolor{red}{in June 1916}, while advancing across an open plain on the company's way to Vimy Ridge. In \textcolor{red}{late 1916}, He had been promoted from sergeant and given a smaller objective position in the near-by village of Meadowhall. His company was soon attacked by a cavalry squad of such ferocity that the archers and machine gunners could have opened up at once. \textcolor{blue}{Little} was bitten, painfully, by a cavalry lance. He received the nicknames "Noire" and "The Snake", because of his long knife, thin face, and reptilian nature. He later hired a pet donkey to pull his belongings back to camp at the end of each day's march. \textcolor{red}{In late May 1917}, three weeks after the Heed affair, \textcolor{blue}{Little} was deployed for the first time in front of the Hindenburg Line when, as part of the Anglo-French alliance, his squad was to capture and destroy German flak guns in Belgium. As if on cue, two machine gun crews came through a broken field and attacked \textcolor{blue}{Little}'s men, who retreated through a canal. }\\
\midrule
\textnormal{\textbf{\textit{Prefix}}:  The total population size of the \textcolor{blue}{African river martin} is unknown. In the late 1980s, it appeared to be common, if local, and large numbers were
}\\\midrule
\textnormal{\textbf{\textit{Human Continuation}}:  seen on migration in Gabon. However, it is particularly poorly studied in the Democratic Republic of the Congo (DRC), and it is not known if there is any relationship between the birds breeding in the DRC and those breeding in coastal areas of Gabon and Congo. Several hundred birds were seen at Conkouati-Douli National Park in Congo \textcolor{red}{in 1996} and a flock of 15,000 birds was seen in Gabon \textcolor{red}{in 1997}. Investigations in Gabon \textcolor{red}{in September 2003} extended the known range of this martin. More than 300 birds were found with hundreds of rosy bee-eaters in the Omboué area and the newly established Iguéla National Park, and a mixed flock with rosy bee-eaters at Loango National Park in Gabon was estimated at 100,000 birds in total. \textcolor{red}{In January 2010}, 250 martins were seen at a new location in Bakoumba, and single birds were observed in hirundine flocks near Moanda, at Mounana, and at Lekoni. 
}\\
\midrule
\textnormal{\textbf{\textit{Adaptive decoding}}:  seen during spring breeding in the Cape province of South Africa. The species is also common in South America, but these differences between species and range regions is unclear. The \textcolor{blue}{African river martin} \textcolor{green!30!black}{is a robust and broad-bodied species that has a dark body and light face, which enhances its camouflage. Its chest is often painted or spotted with black. The coloration and markings can differ between individuals. They have very large ears and a long tail, and in female martins have a pattern that resembles a dorsal stripe, which appears more yellow than brown when united with dark bars, thus giving them the common name of yellow-backed or yellow-fronted. The tail is large, and in males has a faint red pink spot at the base that is very difficult to see unless the animal is under a light. The neck varies from strong to weak, and males often have longer teeth.}
}\\
\midrule
\textnormal{\textbf{\textit{Prefix}}:   The water temperature of \textcolor{blue}{Briar Creek} near its source ranges from 35.96 to 64.94°F (2.20 to 18.30°C), }\\\midrule
\textnormal{\textbf{\textit{Human Continuation}}:  with a median of 51.44°F (10.80°C). This is the highest minimum temperature of any studied location in the watershed. Near West Branch \textcolor{blue}{Briar Creek}'s confluence with \textcolor{blue}{Briar Creek}, the latter creek has a temperature range of 32 to 67.46°F (0.00 to 19.70°C) and a median of 49.28°F (9.60°C). This is the lowest minimum temperature of any studied location in the drainage basin. Near its mouth, the temperature ranges from 32.18 to 80.96°F (0.10 to 27.20°C). The median temperature is 57.56°F (14.20°C). On West Branch \textcolor{blue}{Briar Creek}, the temperature is a minimum of 35.6°F (2.0°C), a median of 57.38°F (14.10°C), and a maximum of 71.6°F (22.0°C). Two miles downstream of East Branch \textcolor{blue}{Briar Creek}'s headwaters, the median temperature is 49.73°F (9.85°C) and the range is 32.18 to 71.42°F (0.10 to 21.90°C). 
}\\
\midrule
\textnormal{\textbf{\textit{Adaptive decoding}}:  with a seasonal mean range of 40.52°F (3.68°C). \textcolor{red}{Seasonal} variation in temperature, as well as production of sulfides, can significantly affect food source quality, quality of water, and the quality and quantity of sediment that may be trapped in the porous pipes and culverts that lead into groundwater.
Relatively constant temperatures between 40°F (4°C) and 55°F (13°C) are common in the watershed. Temperatures above these extremes may occur occasionally when the watershed is dry, but the average annual maximum temperature in \textcolor{blue}{Briar Creek} is 52.92°F (12.55°C). Average \textcolor{red}{summer} temperatures on the \textcolor{blue}{creek}'s western slopes are generally less than 60°F (15°C) \textcolor{red}{throughout most of the year}. Sustaining temperature above 60°F (15°C) is usually limited to the \textcolor{red}{summer months} and are relatively infrequent.
\textcolor{blue}{Briar Creek} drains nine watersheds with areas in the northern half of the watersheds containing more than 70\% of the plants (wheat, barley, rice, rye and bulgur), but the predominant tree species (huckleberry, aspen, birch, and pine) is oak. 
}\\
\bottomrule
\end{tabular}
\end{sc}
\end{small}
\end{center}
\vskip -0.1in
\end{table*}

\begin{table*}[t]
\caption{Case study: Text continuation with the Llama2-7B model. Entities are in blue fonts, logical relationships are in red, and coherent contents are highlighted in green.
}
\label{case_llama}
\vskip 0.15in
\begin{center}
\begin{small}
\begin{sc}
\begin{tabular}{p{16cm}}
\toprule
\textnormal{\textbf{\textit{Prefix}}:  \textcolor{blue}{Liszt} found his method through two compositional practices, which he used in his \textcolor{green!30!black}{symphonic poems}. \textcolor{red}{The first practice} was cyclic form, a}\\\midrule
\textnormal{\textbf{\textit{Human Continuation}}: procedure established by Beethoven in which certain movements are not only linked but actually reflect one another's content.  \textcolor{blue}{Liszt} took Beethoven's practice one step further, combining separate movements into a single-movement cyclic structure.  Many of \textcolor{blue}{Liszt}'s mature works follow this pattern, of which Les préludes is one of the best-known examples. \textcolor{red}{The second practice} was thematic transformation, a type of variation in which one theme is changed, not into a related or subsidiary theme but into something new, separate and independent. }\\
\midrule
\textnormal{\textbf{\textit{Adaptive decoding}}: form that repeats the thematic material several times in different \textcolor{green!30!black}{musical settings}, as with the work of Haydn and Beethoven. It also was used in works such as Wagner's \textcolor{green!30!black}{opera cycle}, Der Ring des Nibelungen and Brahms' Alto Rhapsody. \textcolor{red}{Liszt's second practice} was \textcolor{green!30!black}{program music}, which uses musical form in order to support programmatic concepts. \textcolor{blue}{Liszt} applied both compositional practices to his \textcolor{green!30!black}{symphonic poems, using cyclic form to construct musical episodes that reflect programmatic ideas. The music reflects the program throughout the piece.} }\\
\midrule
\textnormal{\textbf{\textit{Prefix}}:  \textcolor{green!30!black}{Rebellions} can occur if the \textcolor{blue}{loyalty} of a particular province falls too low, with a rebel army appearing in the province to attempt to assume control
}\\\midrule
\textnormal{\textbf{\textit{Human Continuation}}: from the owners. \textcolor{green!30!black}{Civil wars} may also take place if several generals commanding large armies have sufficiently low \textcolor{blue}{loyalty}. \textcolor{green!30!black}{In the event of a civil war, the player is given the choice to back either the current rulers or the rebels.} It had been planned to allow other factions who had established a prior claim to the throne by marriage to princesses to join in a civil war to claim the throne for themselves; however, this was never implemented. Naval warfare is carried out upon the campaign map, where ships can be built and organised into fleets. These fleets can be used to control the game's sea regions and form sea lanes, allowing trade and troop movement between provinces that have constructed a port. Fleets can engage in sea battles with foreign fleets, although unlike land battles these are resolved by the computer. Religion plays an important aspect in Medieval: \textcolor{green!30!black}{Total War}, with religious differences between the Catholic, Orthodox and Muslim factions affecting diplomacy and population \textcolor{blue}{loyalty}. 
 }\\
\midrule
\textnormal{\textbf{\textit{Adaptive decoding}}: . If the \textcolor{blue}{loyalty} falls too far, the province will secede, and become a separate, unallied nation (this may sometimes be undesired, so it is important to know how to increase \textcolor{blue}{loyalty} and to always pay attention to the needs of all provinces). \textcolor{green!30!black}{As a province is loyal}, its \textcolor{blue}{loyalty} will be decreased by any war that the province takes part in, and will be increased by having any kind of benefit in the province.
As a result of the fall of the West Roman Empire, \textcolor{green!30!black}{most of the Western Mediterranean provinces were either conquered or claimed} by the Christian Roman States, or by Germanic tribes such as the Visigoths, Ostrogoths or Vandals. The Eastern Mediterranean however is largely in the hands of Arabic Caliphates, or the Empire of Japan.
The Western Mediterranean is the primary area where player wars begin, although it is possible to war in the eastern mediterranean, the arctic, or most recently, in the Indian Ocean. The Western Mediterranean however is also the richest area in the game, \textcolor{green!30!black}{with a large amount of resources and a large number of wealthy cities which can be attacked.}
}\\
\midrule
\textnormal{\textbf{\textit{Prefix}}:  In 1947 \textcolor{blue}{Bessie} was elected to the Labour Party's National Executive Committee (NEC). She was generally identified with the left
}\\\midrule
\textnormal{\textbf{\textit{Human Continuation}}: wing of the party, and was for a time associated with a grouping known as the "Socialist Fellowship", which espoused a programme of colonial freedom, workers' control and reduced arms expenditure. She resigned from the Fellowship in 1950, along with fellow-MPs Fenner Brockway and Ellis Smith, when it condemned the United Nations intervention in the Korean War. She continued her wholehearted campaigning on behalf of the poorest in the country, pleading with parliament to "remember the queues outside the Poor Relief offices", and \textcolor{red}{castigating the "New Look" fashion of 1948 as wasteful}, "the ridiculous whim of idle people". \textcolor{blue}{Bessie}'s fiery reputation did not harm her electorally; \textcolor{red}{in the February 1950 general election}, with the Exchange constituency greatly increased by boundary changes, her majority rose to 5,344. Nationally Labour lost 76 seats, and its parliamentary majority was reduced to five. Attlee's second government was short-lived; 
 }\\
\midrule
\textnormal{\textbf{\textit{Adaptive decoding}}: wing of the Labour Party. She was also active in organisations supporting Jewish refugees from the Nazis.
With the help of another Jewish member of the Labour Party, Maurice Jacobson, \textcolor{blue}{Bessie} had her own radio programme (\textcolor{blue}{Bessie} Braddock's Book) and also became a popular speaker at meetings. \textcolor{red}{In 1947} she was approached by the National Union of General and Municipal Workers (NUGM) and asked if she would run for parliament in an election.
\textcolor{red}{In the general election of 1950} \textcolor{blue}{Bessie} Braddock was elected as Member of Parliament for Liverpool Exchange. She stood for the National Union of General and Municipal Workers, Labour Co-Operative and the Socialist Medical Association (of which she was a founder member). Her election campaign became more and more left wing.
In Parliament she was very outspoken and unpredictable, using colourful language in her speeches. A political opponent once described her as "the only man in the House of Commons".
\textcolor{red}{In 1954} she was appointed as parliamentary private secretary to Hugh Gaitskell.
}\\
\bottomrule
\end{tabular}
\end{sc}
\end{small}
\end{center}
\vskip -0.1in
\end{table*}

% You can have as much text here as you want. The main body must be at most $8$ pages long.
% For the final version, one more page can be added.
% If you want, you can use an appendix like this one.  

% The $\mathtt{\backslash onecolumn}$ command above can be kept in place if you prefer a one-column appendix, or can be removed if you prefer a two-column appendix.  Apart from this possible change, the style (font size, spacing, margins, page numbering, etc.) should be kept the same as the main body.
%%%%%%%%%%%%%%%%%%%%%%%%%%%%%%%%%%%%%%%%%%%%%%%%%%%%%%%%%%%%%%%%%%%%%%%%%%%%%%%
%%%%%%%%%%%%%%%%%%%%%%%%%%%%%%%%%%%%%%%%%%%%%%%%%%%%%%%%%%%%%%%%%%%%%%%%%%%%%%%

\end{document}